\documentclass[letterpaper]{article} 
\usepackage{aaai23}  
\usepackage{times}  
\usepackage{helvet}  
\usepackage{courier}  
\usepackage[hyphens]{url}  
\usepackage{graphicx} 
\usepackage{natbib}  
\usepackage{caption} 
\usepackage{algorithm}
\usepackage{algorithmic}

\usepackage{newfloat}
\usepackage{listings}
\DeclareCaptionStyle{ruled}{labelfont=normalfont,labelsep=colon,strut=off} 
\floatstyle{ruled}
\newfloat{listing}{tb}{lst}{}
\floatname{listing}{Listing}
\usepackage{color}  

\title{A Generalized Scalarization Method for Evolutionary Multi-Objective Optimization}
\author{
    Ruihao Zheng\textsuperscript{}, 
    Zhenkun Wang\textsuperscript{}\thanks{Corresponding author.}
}
\affiliations{
    \textsuperscript{} School of System Design and Intelligent Manufacturing, Southern University of Science and Technology, Shenzhen, China\\
    12132686@mail.sustech.edu.cn, wangzhenkun90@gmail.com
    
}

\usepackage{bibentry}

\usepackage{amsmath, amsthm, amssymb, enumerate, xcolor}
\usepackage{subfigure}
\usepackage{multirow, booktabs, colortbl}  
\def\ie{\emph{i.e.}}
\def\eg{\emph{e.g.}}

\renewcommand{\figurename}{Figure}
\newtheorem{myDef}{Definition}
\newtheorem{myThe}{Theorem}

\begin{document}

\maketitle

\begin{abstract}
The decomposition-based multi-objective evolutionary algorithm (MOEA/D) transforms a multi-objective optimization problem (MOP) into a set of single-objective subproblems for collaborative optimization. Mismatches between subproblems and solutions can lead to severe performance degradation of MOEA/D. Most existing mismatch coping strategies only work when the $L_{\infty}$ scalarization is used. A mismatch coping strategy that can use any $L_{p}$ scalarization, even when facing MOPs with non-convex Pareto fronts, is of great significance for MOEA/D. This paper uses the global replacement (GR) as the backbone. We analyze how GR can no longer avoid mismatches when $L_{\infty}$ is replaced by another $L_{p}$ with $p\in [1,\infty)$, and find that the $L_p$-based ($1\leq p<\infty$) subproblems having inconsistently large preference regions. When $p$ is set to a small value, some middle subproblems have very small preference regions so that their direction vectors cannot pass through their corresponding preference regions. Therefore, we propose a generalized $L_p$ (G$L_p$) scalarization to ensure that the subproblem's direction vector passes through its preference region. Our theoretical analysis shows that GR can always avoid mismatches when using the G$L_p$ scalarization for any $p\geq 1$. The experimental studies on various MOPs conform to the theoretical analysis.

\end{abstract}

\section{Introduction}

The multi-objective optimization problem (MOP) can be written as
\begin{equation}\label{MOP}
    \begin{aligned}
        \mbox{minimize}\quad & \mathbf{f}(\mathbf{x}) = (f_1(\mathbf{x}),\ldots,f_m(\mathbf{x}))^{\intercal}, \\
        \mbox{subject to}\quad & \mathbf{x} \in \Omega,
    \end{aligned}
\end{equation}
where $\mathbf{x} = (x_1,\ldots, x_n)^{\intercal}$ is a decision vector (also called solution), and $\Omega \subset \mathbb{R}^n$ denotes the decision space. $\mathbf{f}: \mathbb{R}^n \rightarrow \mathbb{R}^m$ is composed of $m$ objective functions, and $\mathbf{f}(\mathbf{x})$ is the objective vector corresponding to $\mathbf{x}$.

The multi-objective evolutionary algorithm based on decomposition (MOEA/D) is a popular framework for dealing with MOPs~\cite{zhang2007moea}. It converts an MOP into a set of single-objective subproblems to optimize them simultaneously. The subproblem function is defined by a scalarization method, where the family of the $L_p$ ($p\geq 1$) scalarization is often used. The weighted sum (WS) method and the Tchebycheff (TCH) method are the $L_1$ scalarization and the $L_\infty$ scalarization, respectively. Any scalarization method can be used in MOEA/D, and each has its own strengths and weaknesses~\cite{hansen2000use,wang2016decomposition}.

In MOEA/D, each subproblem is associated with one solution. For each iteration, the evolutionary operations are conducted with respect to each subproblem to generate a new solution; this new solution is used to replace several neighboring subproblems' original solutions if it is better than these solutions. As stated in \cite{wang2014replacement,li2013stable}, there may exist mismatches between solutions and subproblems and the above replacement strategy can lead to severe performance degradation. Sequentially, several MOEA/D variants with mismatch coping strategies have been proposed, such as MOEA/D-GR~\cite{wang2014replacement}, MOEA/D-STM~\cite{li2013stable}, MOEA/D-IR~\cite{li2014interrelationship}, MOEA/D-AMOSTM~\cite{wu2017matching}, and MOEA/D-2TCHMFI~\cite{ma2017tchebycheff}. Most of these algorithms employ the TCH method to define subproblems. Because the TCH-based subproblem has a good property, \ie, the intersection between its direction vector and the Pareto front is optimal for it. Nevertheless, the TCH method has some specific weaknesses compared to other $L_p$ scalarization methods. It is non-smooth, non-differentiable, and may cause slow convergence of MOEA/D, making it difficult to use in many scenarios. Therefore, it is of great significance for MOEA/D to develop mismatch coping strategies that enable using any $L_p$ scalarization method, even when facing MOPs with non-convex Pareto fronts.

This paper adopts MOEA/D-GR as the backbone. It matches the most suitable subproblem for each new solution according to the function values over all subproblems. Our analysis first reveals that MOEA/D-GR can avoid mismatches only when the TCH method (\ie, $L_\infty$) is used. When the TCH method is replaced with another $L_p$ scalarization for any $p\in [1,\infty)$, MOEA/D-GR fails to choose the appropriate subproblem for each solution. For example, MOEA/D-GR always chooses boundary subproblems to update if the TCH method is substituted with the WS method. Our analysis demonstrates that these mismatches can be attributed to $L_p$-based ($1\leq p<\infty$) subproblems having inconsistently large preference regions. When $p$ is set to a small value, the corresponding subproblems have preference regions with extremely imbalanced sizes. The boundary subproblem's preference region is much larger than that of the middle subproblem (as shown in \figurename~\ref{fig:preference_region_NS}). Such an imbalance leads to severe mismatches in MOEA/D-GR. 

To fill this gap, we propose a generalized $L_p$ (G$L_p$) scalarization for the subproblem definition. The G$L_p$-based subproblems can have uniform preference regions, no matter what the value of $p$ ($p\geq 1$) is set to. We apply the G$L_p$ scalarization to MOEA/D-GR and term the new algorithm as MOEA/D-GGR. The effectiveness of MOEA/D-GGR is validated with different $p$-values on various continuous and combinatorial MOPs. The results indicate that our method can avoid mismatches using the scalarization of any norm (\ie, $p$ can be $1,2\ldots,\infty$), even in dealing with MOPs with non-convex or other complex Pareto fronts. Our method significantly expands the applicability of MOEA/D.


\section{Background}
\subsection{Basic Concepts}

\begin{myDef}
For two objective vectors $\mathbf{u} = (u_1,\ldots, u_m)^{\intercal}$ and $\mathbf{v} = (v_1,\ldots, v_m)^{\intercal}$, $\mathbf{u}$ is said to \textbf{\em dominate} $\mathbf{v}$, if $u_i \leq v_i$ for all $i\in \{1,\ldots,m\}$ and $u_{j}<v_{j}$ for at least one $j\in \{1,\ldots,m\}$.
\end{myDef}

\begin{myDef}
$\mathbf{x}^{*}$ is called the \textbf{\em Pareto optimal} solution, if there is no $\mathbf{x}\in \Omega$ such that $\mathbf{f}(\mathbf{x})$ dominates $\mathbf{f}(\mathbf{x}^{*})$. Correspondingly, $\mathbf{f}(\mathbf{x}^*)$ is called the \textbf{\em non-dominated} vector.
\end{myDef}

\begin{myDef}
The set of all Pareto optimal solutions is called the \textbf{\em Pareto set} (denoted as $PS$), and its image in the objective space is called the \textbf{\em Pareto front} (denoted as $PF$).
\end{myDef}

\begin{myDef}
The \textbf{\em ideal point} $\mathbf{z}^{ide} = (z^{ide}_1,\ldots,z^{ide}_m)^\intercal$ is defined as $z^{ide}_i = \mathop{\min}\{{f_{i}}(\mathbf{x})|\mathbf{x}\in\Omega\}$, for $i=1,\ldots,m$.
\end{myDef}

\begin{myDef}
The \textbf{\em utopian point} $\mathbf{z}^{uto} = (z^{uto}_1,\ldots,z^{uto}_m)^\intercal$ is defined as $z^{uto}_i = z^{ide}_i-\varepsilon_i$, for $i=1,\ldots,m$, where $\varepsilon_i>0$ is a relatively small computationally significant scalar.
\end{myDef}

\subsection{$L_p$ Scalarization}
The $L_p$ scalarization defines the single-objective optimization subproblem as
\begin{equation}\label{eqn:p-norm}
    g^{lp}(\mathbf{x}|\mathbf{w},\mathbf{z}^*) = \left(\sum_{i=1}^m (w_i|f_i(\mathbf{x})-z^*_i|)^p\right)^{\frac{1}{p}},
\end{equation}
where $\mathbf{w}=(w_1,\ldots,w_m)^{\intercal}$ is a weight vector that satisfies $w_i\geq0$ for each $i \in \{1,\ldots,m\}$ and $\sum_{i=1}^m w_i=1$. The reference vector $\mathbf{z}^{*} = (z^{*}_1,\ldots,z^{*}_m)^\intercal$ is usually set to $\mathbf{z}^{ide}$. By using a set of uniformly distributed weight vectors $\{\mathbf{w}^j\}_{j=1}^{N}$ in Eq.~(\ref{eqn:p-norm}), $N$ single-objective subproblems can be achieved.

Let $\mathbf{z}^*=\mathbf{z}^{ide}$ and $p=1$, Eq. \eqref{eqn:p-norm} can be simplified as
\begin{equation}\label{eqn:WS}
    g^{ws}(\mathbf{x}|\mathbf{w},\mathbf{z}^*) = \sum_{i=1}^m w_i (f_i(\mathbf{x})-z_i^*).
\end{equation}
Eq. \eqref{eqn:WS} is also known as the WS method. When $p\rightarrow \infty$, Eq. \eqref{eqn:p-norm} can be written as
\begin{equation}\label{eqn:TCH}
    g^{tch}(\mathbf{x}|\mathbf{w},\mathbf{z}^*) = \max_{i\in\{1,\ldots,m\}} w_i(f_i(\mathbf{x})-z^*_i).
\end{equation}
Eq. \eqref{eqn:TCH} is also known as the TCH method. For any Pareto optimal solution $x^*$, there exists a weight vector such that $x^*$ is the optimal solution of the corresponding TCH-based subproblem~\cite{miettinen2012nonlinear}. For the subproblem defined by Eq. \eqref{eqn:TCH} with a weight vector $\mathbf{w}^{j}$, we refer $\boldsymbol{\lambda}^j=(\frac{1}{w_1^j},\ldots,\frac{1}{w_i^j})^{\intercal}$ to as its direction vector. The intersection between $\boldsymbol{\lambda}^j$ and the $PF$ is the optimal objective vector of this subproblem~\cite{qi2014moea}.

\subsection{MOEA/D Framework}
MOEA/D employs $N$ uniformly distributed weight vectors $\{\mathbf{w}^j\}_{j=1}^{N}$ to generate $N$ subproblems. MOEA/D calculates the Euclidean distance between every two subproblems' weight vectors and uses these distances to define the mating and replacement neighborhoods for each subproblem. MOEA/D maintains a population with $N$ solutions $\mathbf{x}^1,\ldots,\mathbf{x}^N$, where $\mathbf{x}^j$ is associated with the $j$-th subproblem for $j=1,\ldots,N$. The mating or replacement neighborhood of $\mathbf{x}^j$ (denoted as $B_m^j$ or $B_r^j$) consists of the solutions of the $j$-th subproblem's the $T_m$ or $T_r$ closest neighbors. 
At each iteration, MOEA/D conducts operations with respect to the $j$-th subproblem for each $j\in\{1,\ldots,N\}$ as follows:
\begin{enumerate}[a)]
    \item Conduct reproduction operators on solutions randomly selected from $B_m^j$ to generate a new solution $\mathbf{x}^{new}$.
    \item For each solution $\mathbf{x}^k$ of $B_r^j$, replace $\mathbf{x}^k$ by $\mathbf{x}^{new}$ if $\mathbf{x}^{new}$ is better than it.
\end{enumerate}

\subsection{GR}
In MOEA/D, the new solution is only allowed to update the solutions of the neighboring subproblems. As shown in \figurename~\ref{fig:why_GR}, the new solution of the $4$-th subproblem is only allowed to update the $4$-th subproblem's three neighboring subproblems' solutions (\ie, solutions of the $3$-rd to the $5$-th subproblems). However, this replacement may cause mismatches between subproblems and solutions, thereby severely hindering the algorithm's performance. For example, the new solution in \figurename~\ref{fig:why_GR} cannot benefit the neighbors of the $4$-th subproblem but can facilitate the convergence of the $7$-th and $8$-th subproblems.

\begin{figure}[ht]
\centering
\includegraphics[width=0.4\textwidth]{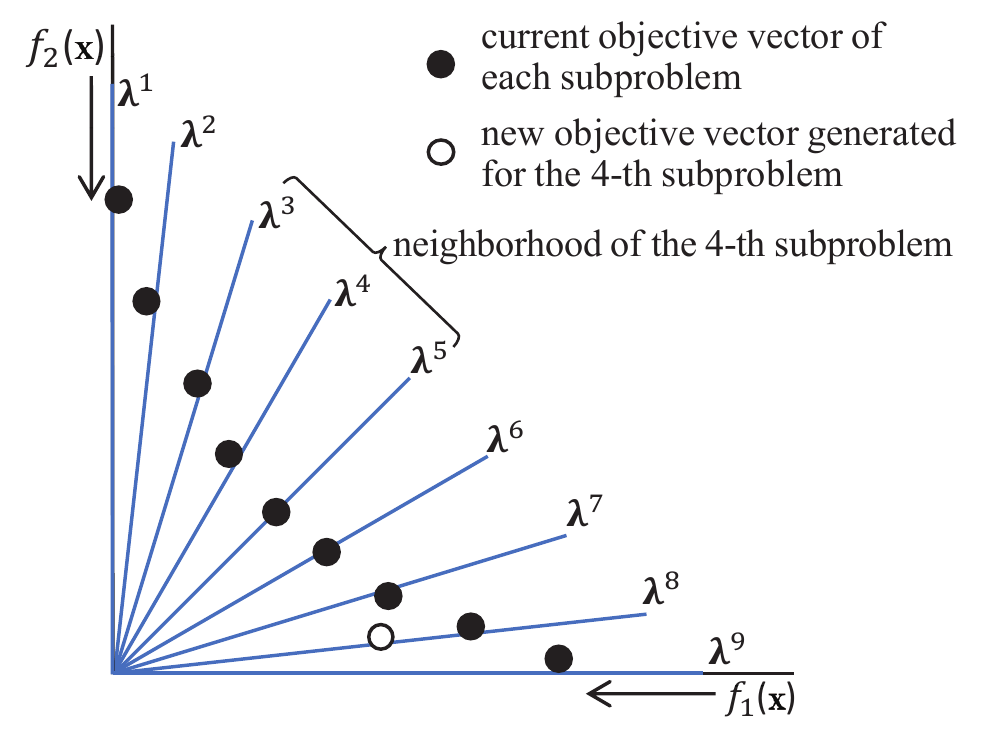}
\caption{A case of mismatch in MOEA/D.}
\label{fig:why_GR}
\end{figure}

GR globally selects suitable updating subproblems for each new solution. For the new solution $\mathbf{x}^{new}$, GR first locates the most appropriate subproblem via
\begin{equation}
 \begin{array}{rl}
 j=\mathop{\arg\min}\limits_{k\in\{1,\ldots,N\}}\{g^{tch}(\mathbf{x}^{new}|\mathbf{w}^k,\mathbf{z}^*)\}.
 \end{array}
 \label{mSub}
\end{equation}
Thereafter, the solutions of subproblems within the neighborhood of the $j$-th subproblem are assigned to $\mathbf{x}^{new}$ for updating. Since MOEA/D-GR utilizes the new solution's subproblem function values to determine its most appropriate subproblem, the choice of the scalarization method that defines the subproblem function is critical.

\begin{figure*}[t]
	\centering
	\subfigure[$L_1$]{
	    \includegraphics[width=0.29\textwidth]{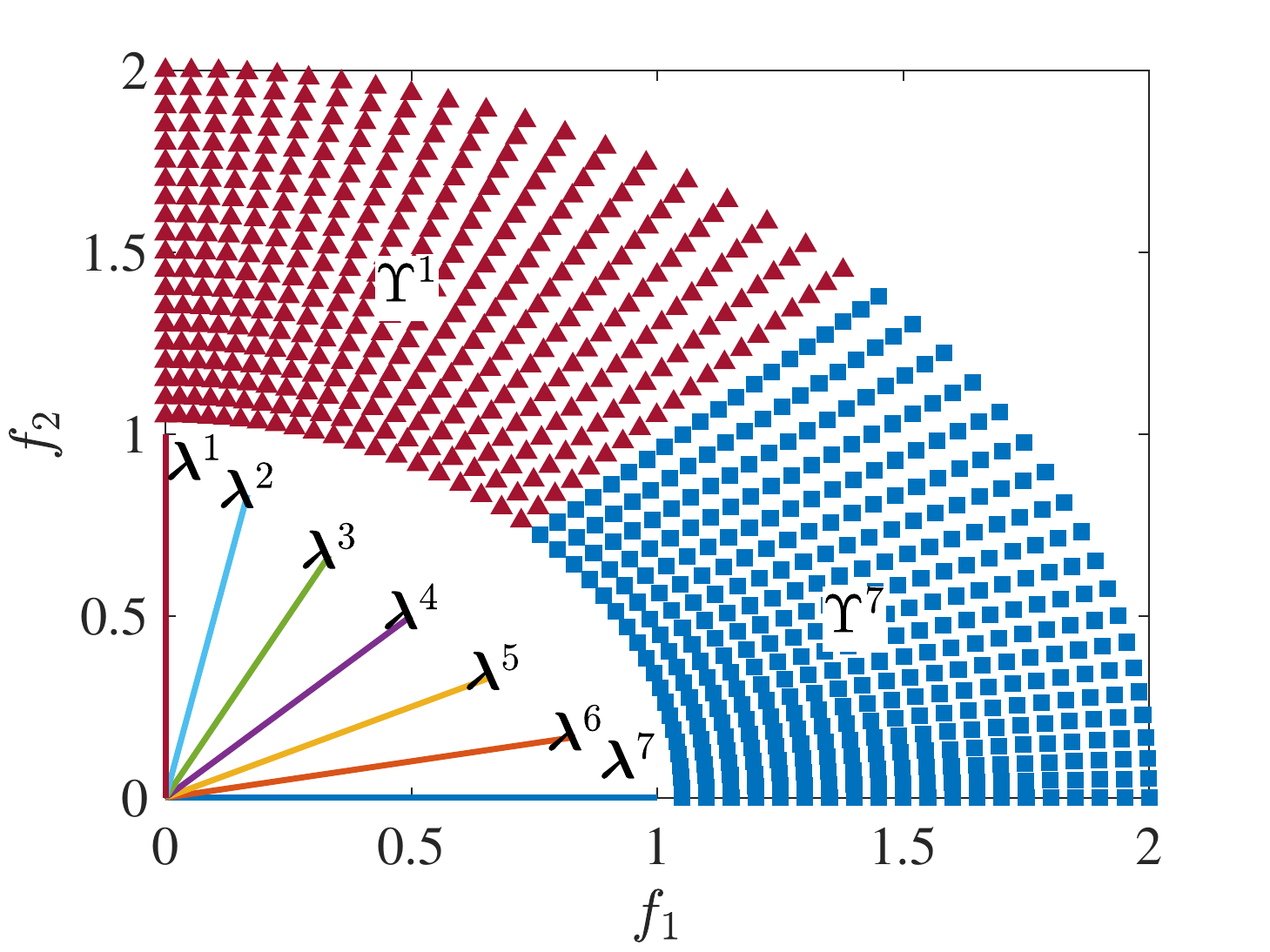}\label{subfig:p1_preference}
	}
	\subfigure[$L_2$]{
	    \includegraphics[width=0.29\textwidth]{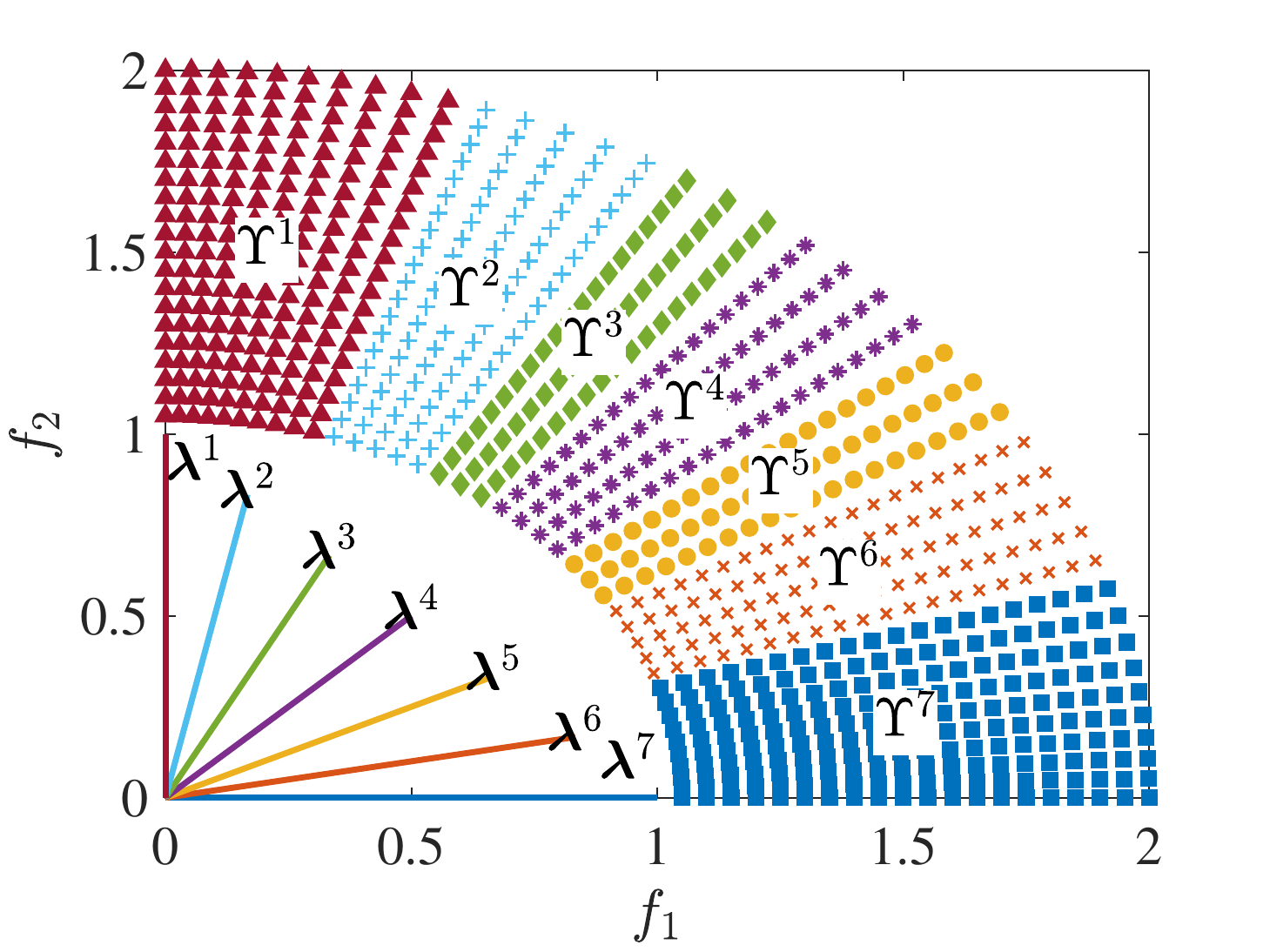}\label{subfig:p2_preference}
	}
	\subfigure[$L_{\infty}$]{
	    \includegraphics[width=0.29\textwidth]{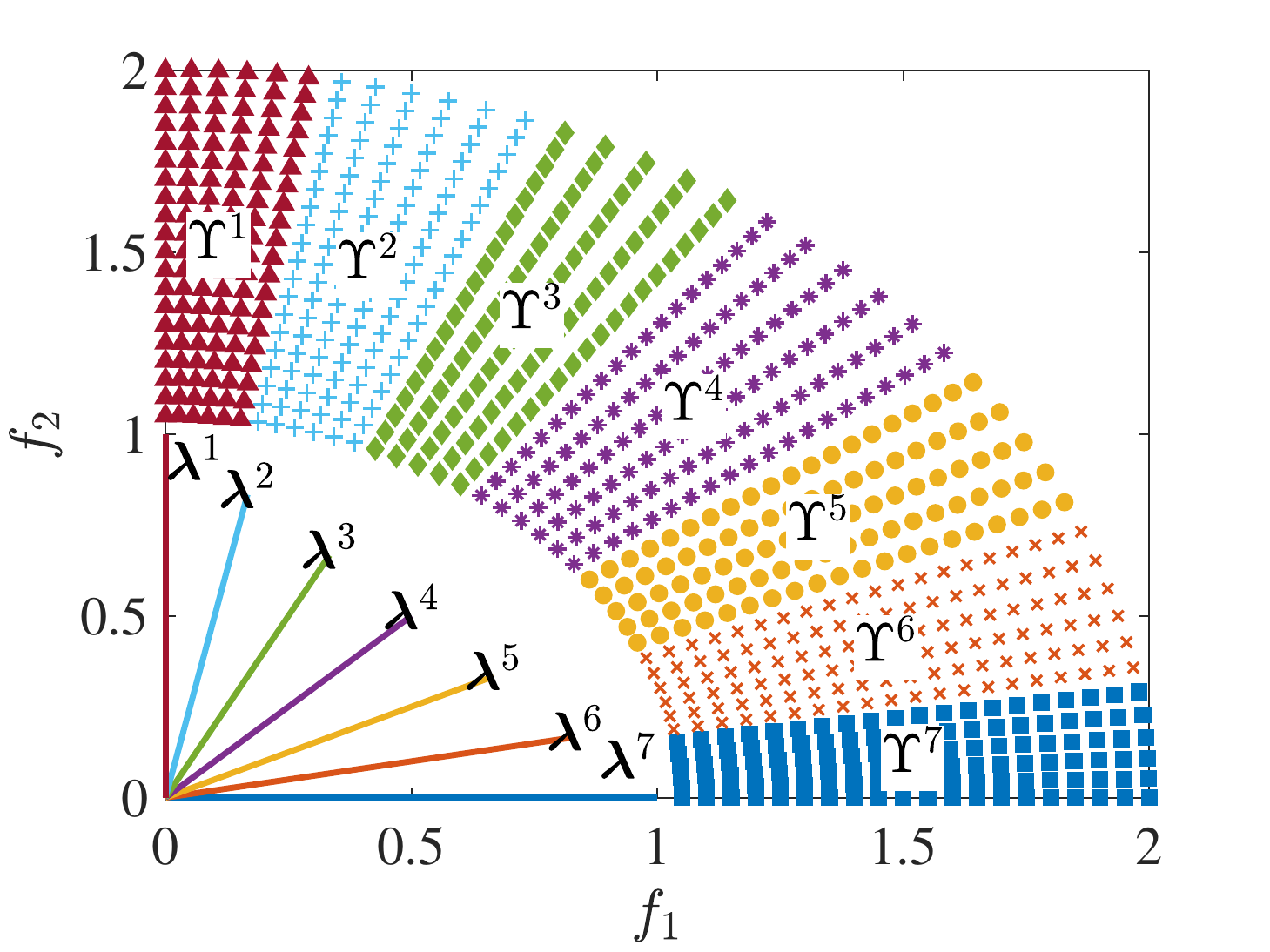}\label{subfig:pinf_preference}
	}
	\caption{Preference regions of $L_p$-based subproblems with $\mathbf{z}^*=(0,0)^{\intercal}$ and $\{\mathbf{f}=(f_1,f_2)^{\intercal}|1\leq \lVert \mathbf{f} \rVert_2 \leq 2\}$.}
	\label{fig:preference_region_NS}
\end{figure*}

\begin{figure*}[t]
	\centering
	\subfigure[G$L_1$]{
	    \includegraphics[width=0.29\textwidth]{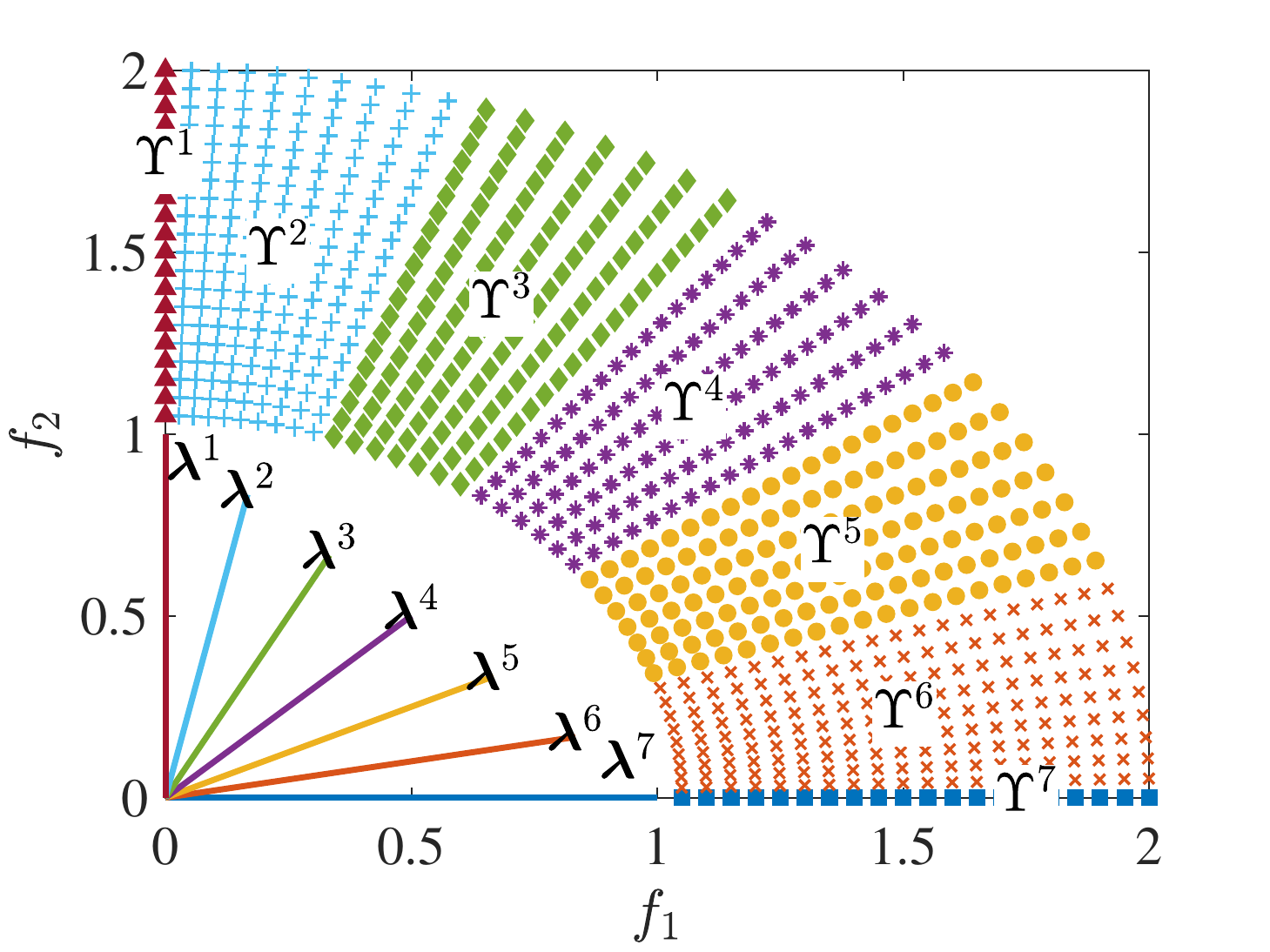}\label{subfig:GNSp1_preference}
	}
	\subfigure[G$L_2$]{
	    \includegraphics[width=0.29\textwidth]{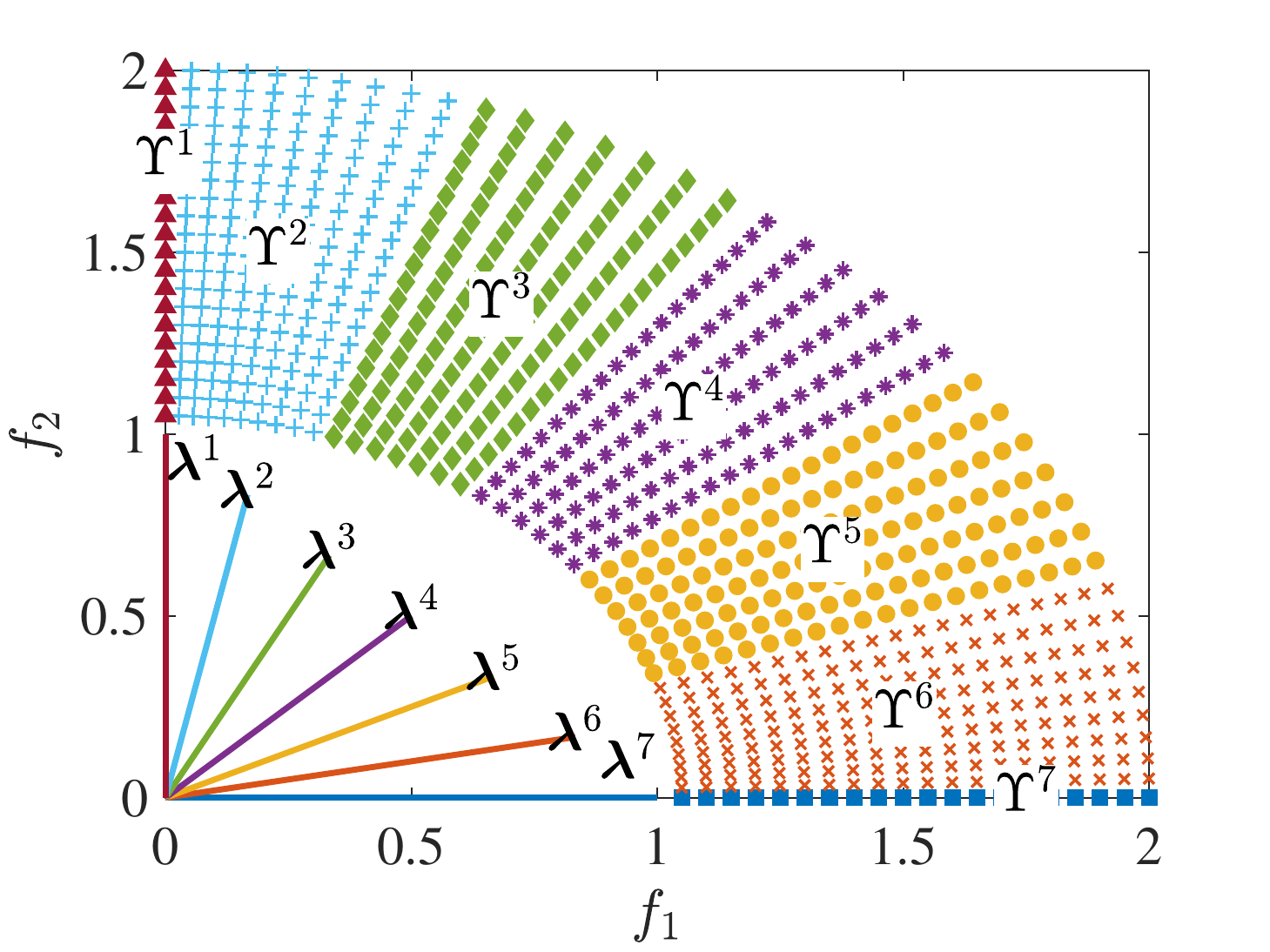}\label{subfig:GNSp2_preference}
	}
	\subfigure[G$L_{\infty}$]{
	    \includegraphics[width=0.29\textwidth]{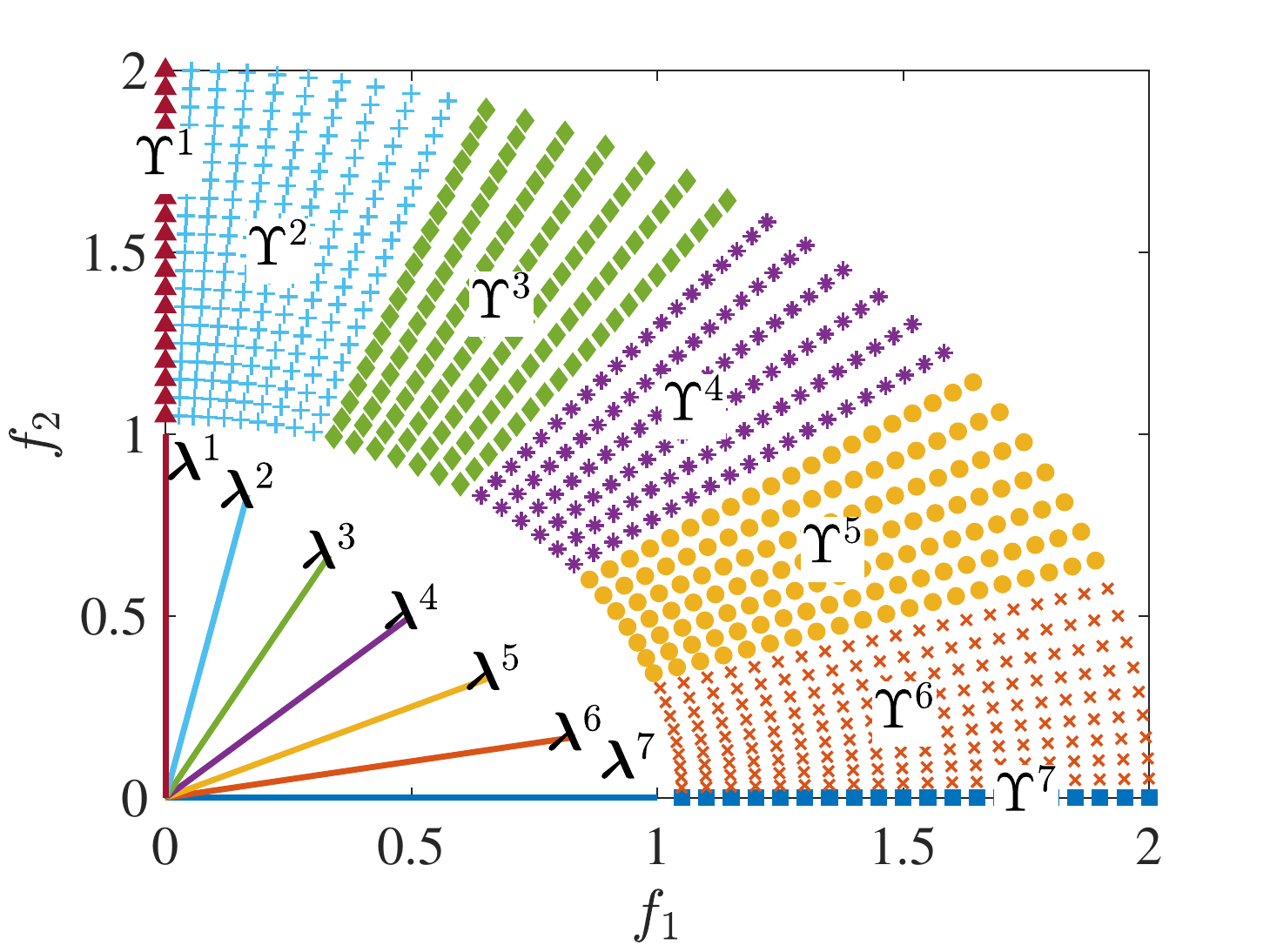}\label{subfig:GNSpinf_preference}
	}
	\caption{Preference regions of G$L_p$-based subproblems with $\mathbf{z}^*=(0,0)^{\intercal}$ and $\{\mathbf{f}=(f_1,f_2)^{\intercal}|1\leq \lVert \mathbf{f} \rVert_2 \leq 2\}$.}
	\label{fig:preference_region_GNS}
\end{figure*}

\section{Generalized $L_p$ Scalarization}
\subsection{Motivation}\label{moti}
According to ~\cite{ma2017tchebycheff}, the preference region of the $j$-th subproblem can be described as
\begin{equation}
    \Upsilon^j = \left\{ \mathbf{f}(\mathbf{x}) | \mathbf{x}\in \Omega, \mathop{\arg\min}_{k\in \{1,\ldots,N\}}\{g(\mathbf{x}|\mathbf{w}^k, \mathbf{z}^*)\}=j \right\}.
\end{equation}
The preference regions of the $L_p$-based subproblems are illustrated in \figurename~\ref{fig:preference_region_NS}.
When $p=1$, only the $1$-st and $7$-th subproblems have the preference regions and the other subproblems show no preference region. When $p=2$, all the subproblems have preference regions, but some subproblems' direction vectors (\eg, $\boldsymbol{\lambda}^2,\boldsymbol{\lambda}^3,\boldsymbol{\lambda}^5,\boldsymbol{\lambda}^6$) do not pass through their corresponding regions.


\begin{myDef}
A subproblem is called \textbf{\em boundary subproblem} if $\mathbf{w}$ has at least one minimal entry. A subproblem is called \textbf{\em extreme boundary subproblem} if $\mathbf{w}$ has $(m-1)$ minimal entries.
\end{myDef}

\begin{myThe}\label{the:p1_preference}
For $L_1$-based subproblems, only the extreme boundary subproblems have preference regions.
\end{myThe}

\begin{proof}
Without loss of generality, we assume $z_i^*=0$, $f_i\in\mathbb{R}_{\geq0}$ and $w_i\in\mathbb{R}_{\geq0}$ for $i=1,\ldots,m$. The subproblem of a given preference objective vector $\mathbf{f}$ can be determined by solving the following linear programming problem
\begin{equation}
    \begin{aligned}
        & \mathop{\mbox{minimize}}\limits_{\mathbf{w}} \quad \mathbf{w}^{\intercal} \mathbf{f}, \\
        & \mbox{subject to} \quad \left\{
        \begin{aligned}
            & A\mathbf{w} = b, \\
            & w_i\geq 0 \mbox{ for } i=1,\ldots,m,
        \end{aligned}\right.
    \end{aligned}
\end{equation}
where $A=[1 \cdots 1]^{1 \times m}$ and $b=1$. The optimal solution to this problem is one of the basic solutions. Denote $B_k$ as $k$-th column of $A$. Since $B_k^{-1}b=1$ for $k=1,\ldots,m$, then the basic solution $k$ to this problem is
\begin{equation}\label{eqn:WS_GR_bas}
    w_i^{{bas}_k} = \left\{
    \begin{aligned}
        & 1, i = k, \\
        & 0, i\neq k.
    \end{aligned}\right.
\end{equation}
Eq. \eqref{eqn:WS_GR_bas} represents that the subproblem of a given preference objective vector $\mathbf{f}$ is always one of the extreme boundary subproblems.
\end{proof}


\begin{myThe}\label{the:p>1_preference}
The direction vectors of $L_p$-based ($1\leq p<\infty$) subproblems are not guaranteed to pass through the corresponding preference regions except direction vector $(m,\ldots,m)^{\intercal}$. The direction vectors of $L_{\infty}$-based subproblems all pass through the corresponding preference regions.
\end{myThe}

\begin{proof}
We assume $z_i^*=0$, $f_i\in\mathbb{R}_{\geq0}$ and $w_i\in\mathbb{R}_{\geq0}$ for $i=1,\ldots,m$. $\sum_{i=1}^m w_i=1$ is substituted into Eq. \eqref{eqn:p-norm} and then we have
\begin{equation}\label{eqn:p-norm_preference}
    g^{\prime}(\mathbf{w}|\mathbf{f},\mathbf{z}^*) \!=\! \left(\sum_{i=1}^{m-1} (w_i f_i)^p \!+\! \left(1\!-\!\sum_{i=1}^{m-1} w_i\right)^p \!f_m^p \right)^{\frac{1}{p}}\!.
\end{equation}
The first-order partial derivative of Eq. \eqref{eqn:p-norm_preference} with respect to $w_k$ is
\begin{equation}\label{eqn:p-norm_preference_d1}
    \frac{\partial g^{\prime}(\mathbf{w}|\mathbf{f},\mathbf{z}^*)}{\partial w_k} = \sigma_1 \cdot \sigma_2,
\end{equation}
where
\begin{equation}
\begin{aligned}
    & \sigma_1 = \frac{1}{p}\left(\sum_{i=1}^{m-1} (w_i f_i)^p + \left(1-\sum_{i=1}^{m-1} w_i\right)^p f_m^p \right)^{\frac{1}{p}-1}, \\
    & \sigma_2 = p\left(w_k^{p-1} f_k^p - \left(1-\sum_{i=1}^{m-1}w_i\right)^{p-1} f_m^p\right).
\end{aligned}
\end{equation}
$\frac{\partial g^\prime(\mathbf{w}|\mathbf{f},\mathbf{z}^*)}{\partial w_k}=0$ if and only if $\sigma_1=0$ or $\sigma_2=0$. First, let $\sigma_1=0$, we can get $w_i f_i=0$ for $i=1,\ldots,m$. Since $w_i \geq 0$ and $f_i \geq 0$, $w_i$ or $f_i$ must be 0. $\sigma_2$ in this case must be 0. Secondly, let $\sigma_2=0$, we can obtain
\begin{equation}\label{eqn:p-norm_preference_d1_0_part}
    w_k^{p-1}f_k^p = \left(1-\sum_{i=1}^{m-1}w_i\right)^{p-1} f_m^p.
\end{equation}
According to Eq. \eqref{eqn:p-norm_preference_d1_0_part}, the solution satisfies
\begin{equation}\label{eqn:p-norm_preference_line}
    \left(\frac{w_k}{w_i}\right)^{p-1} = \left(\frac{f_i}{f_k}\right)^p, ~~i,k\in\{1,\ldots,m\}.
\end{equation}
Then the solution can be written as
\begin{equation}\label{eqn:p-norm_preference_sol}
    w_i = \left(\frac{1}{\alpha f_i}\right)^{\frac{p}{p-1}},\mbox{ for }i=1,\ldots,m,
\end{equation}
where $\alpha\geq 0$ is a constant that ensure $\sum_{i=1}^m w_i = 1$. There exists $A\in\mathbb{R}^{m\times m}$ and $b\in\mathbb{R}^{m}$ which are
\begin{equation}
    \begin{aligned}
        & A = \left[
        \begin{array}{ccccc}
        1     &0      &\cdots&\cdots &0\\
        0       &1 & &       &0\\
        \vdots  &       &\ddots & &\vdots\\
        0       & & &1      &0\\
        -1       &\cdots&\cdots &-1      &0
        \end{array}\right], \\
        & b = (0,\ldots,0,1)^{\intercal},
    \end{aligned}
\end{equation}
such that
\begin{equation}
    g^{lp}(A\mathbf{w}+b|\mathbf{f},\mathbf{z}^*) = g^{\prime}(\mathbf{w}|\mathbf{f},\mathbf{z}^*).
\end{equation}
Since $g^{lp}(\mathbf{w}|\mathbf{f},\mathbf{z}^*)$ is a norm as well as a convex function, Eq. \eqref{eqn:p-norm_preference} is a convex function. Then, Eq. \eqref{eqn:p-norm_preference_sol} is a global minimal solution. 


Eq. \eqref{eqn:p-norm_preference_sol} can be rewritten as $\frac{1}{w_i} = (\alpha f_i)^{\frac{p}{p-1}}$ for $i=1,\ldots,m$. If we have infinitely sampled weight vectors, each subproblem's preference region is a line. When $p\rightarrow\infty$, $\frac{1}{w_i}=\alpha f_i$. The corresponding subproblem's direction vector passes through its preference region. When $p$ takes a value from $(1,\infty)$, only direction vector $(m,\ldots,m)^\intercal$ passes through its preference region while other direction vectors cannot pass through their corresponding preference regions.
\end{proof}

Theorems \ref{the:p1_preference} and \ref{the:p>1_preference} demonstrate that the $L_p$ scalarization with $p\in[1,\infty)$ can cause the performance deterioration of MOEA/D-GR. As shown in \figurename~\ref{fig:pop}, many objective vectors obtained by MOEA/D-GR ($L_1$) are far away from the $PF$. Theorem \ref{the:p1_preference} indicates that MOEA/D-GR ($L_1$) always selects one of the boundary subproblems for a new solution. As a result, the other subproblems cannot be selected for updating. \figurename~\ref{fig:pop} also shows that MOEA/D-GR ($L_{\infty}$) has a better population uniformity than MOEA/D-GR ($L_2$). According to Theorem \ref{the:p>1_preference}, the $L_p$ scalarization with $p\in[1,\infty)$ causes mismatches in MOEA/D-GR and makes MOEA/D-GR fail to achieve a good population uniformity. 

The significance of the above analysis is by no means limited to explaining the mismatches in MOEA/D-GR. As argued in~\cite{hao2017improved}, some mismatches are incurred when MOEA/D-GR ($L_{\infty}$) adopts $\mathbf{z}^{uto}$ instead of $\mathbf{z}^{ide}$ as the reference vector. The reasons can be inferred using Theorem \ref{the:p>1_preference} as well. If the boundary subproblems exceed the feasible objective space, boundary subproblems are never selected for updating in MOEA/D-GR. Moreover, other newly proposed scalarization methods~\cite{jiang2017scalarizing} can also use this analysis of the preference regions to validate if they can deal with mismatches well.


\subsection{Methodology}
The idea of the G$L_p$ scalarization is to modify the $L_p$ scalarization such that any direction vector of a subproblem can pass through its corresponding preference region. The G$L_p$ scalarization is formulated as
\begin{equation}\label{eqn:Mp}
    g^{glp}(\mathbf{x}|\mathbf{w},\mathbf{z}^*) \!=\! \left(\sum_{i=1}^m (w_i(f_i(\mathbf{x})-z^*_i))^p\right)^{\frac{1}{p}} \cdot h(\mathbf{w}),
\end{equation}
where $h(\mathbf{w})$ enables the G$L_p$ scalarization to satisfy the requirement. Note that $h(\mathbf{w})$ only changes the scale of $g^{lp}(\mathbf{x}|\mathbf{w},\mathbf{z}^*)$ among different $\mathbf{w}$. $h(\mathbf{w})$ is a constant for a particular $\mathbf{w}$, and thus the contour shape remains the same. Moreover, the computational effort is low since $h(\mathbf{w})$ can be pre-calculated for each subproblem.

We assume $z_i^*=0$, $f_i\in\mathbb{R}_{\geq0}$ and $w_i\in\mathbb{R}_{\geq0}$ for $i=1,\ldots,m$, the first-order partial derivative of $g^{glp}(\mathbf{w}|\mathbf{f},\mathbf{z}^*)$ with respect to $w_k$ can be calculated as
\begin{equation}\label{eqn:Mp_d1}
\begin{aligned}
    \frac{\partial g^{glp}(\mathbf{w}|\mathbf{f},\mathbf{z}^*)}{\partial w_k} \!=\! & \left(\sum_{i=1}^m (w_i f_i)^p\right)^{\frac{1}{p}-1}w_k^{p-1} \! f_k^p h(\mathbf{w}) + \\
    & \left(\sum_{i=1}^m (w_i f_i)^p\right)^{\frac{1}{p}} \frac{\partial h(\mathbf{w})}{\partial w_k}.
\end{aligned}
\end{equation}
Let $w_i=\frac{1}{\alpha f_i} \mbox{ for } i=1,\ldots,m$ such that $\frac{\partial g^{glp}(\mathbf{w}|\mathbf{f},\mathbf{z}^*)}{\partial w_k}=0$, we can obtain
\begin{equation}\label{eqn:Mp_d1_dir}
    \frac{\partial h(\mathbf{w})}{\partial w_k}=-\frac{1}{mw_k}h(\mathbf{w}).
\end{equation}
According to Eq. \eqref{eqn:Mp_d1_dir}, $h(\mathbf{w})$ can be
\begin{equation}\label{eqn:hw}
    h(\mathbf{w}) = \exp{\left(-\frac{1}{m} \ln{\prod_{i=1}^m w_i}\right)} = \left(\prod_{i=1}^m w_i\right)^{-\frac{1}{m}}.
\end{equation}

\begin{myThe}
The direction vectors of G$L_p$-based ($p\geq 1$) subproblems all pass through the corresponding preference regions.
\end{myThe}

\begin{proof}
We assume $z_i^*=0$, $f_i\in\mathbb{R}_{\geq0}$ and $w_i\in\mathbb{R}_{\geq0}$ for $i=1,\ldots,m$. Let $w_i=\frac{1}{\alpha f_i} + v_i t$ for $i=1,\ldots,m$ where $\mathbf{v}$ is any vector and $t\geq 0$, we have
\begin{equation}\label{eqn:Mp_v}
    g^{glp}(\mathbf{w}|\mathbf{f},\mathbf{z}^*) \!\!=\!\! g^{\prime}(t) \!\!=\!\! \left(\sum_{i=1}^m \left(\frac{1}{\alpha} \!+\! v_i f_i t\right)^p\right)^{\frac{1}{p}} \cdot\! h^{\prime}(t),
\end{equation}
where
\begin{equation}\label{eqn:hw_v}
    h^{\prime}(t) = \exp{\left(-\frac{1}{m} \sum_{i=1}^m \ln{(\frac{1}{\alpha f_i} + v_i t)}\right)}.
\end{equation}
The first-order derivative of Eq. \eqref{eqn:Mp_v} is
\begin{equation}\label{eqn:Mp_v_d1}
    \frac{\partial g^{\prime}(t)}{\partial t} = \sigma_3(\sigma_4 - \sigma_5),
\end{equation}
where
\begin{equation}
    \begin{aligned}
        & \sigma_3 = h^{\prime}(t) \left(\sum_{i=1}^m\left(\frac{1}{\alpha}+v_i f_i t\right)^p\right)^{\frac{1}{p}-1}, \\
        & \sigma_4 = \sum_{i=1}^m \left(\frac{1}{\alpha} + v_i f_i t\right)^{p-1}v_i f_i, \\
        & \sigma_5 = \frac{1}{m} \sum_{i=1}^m\left(\frac{1}{\alpha}+v_i f_i t\right)^p \sum_{i=1}^m \frac{v_i}{\frac{1}{\alpha f_i} + v_i t}.
    \end{aligned}
\end{equation}
First, we consider the sign of $\sigma_4-\sigma_5$. Since $v_i=\frac{1}{t}(w_i-\frac{1}{\alpha f_i})$, we have
\begin{equation}
    \begin{aligned}
        & \sigma_4 = \sum_{i=1}^m v_i w_i^{p-1} f_i^p, \\
        & \sigma_5 = \sum_{i=1}^m \left(\frac{w_i}{m} \sum_{j=1}^m \frac{v_j}{w_j}\right) w_i^{p-1} f_i^p, \\
        & v_i \!\!-\!\! \left(\!\frac{w_i}{m} \sum_{j=1}^m \frac{v_j}{w_j}\!\right)\!=\!\frac{w_i}{t}\left(\!\frac{1}{m}\sum_{j=1}^m \frac{1}{\alpha f_j w_j}\!-\!\frac{1}{\alpha f_i w_i}\!\right).
    \end{aligned}
\end{equation}
Then
\begin{equation}\label{eqn:s2_s3_trans}
\begin{aligned}
    \sigma_4\!-\!\sigma_5 \!=
    & \frac{1}{\alpha t} \sum_{i=1}^m(w_i f_i)^p \left(\frac{1}{m}\sum_{j=1}^m \frac{1}{f_j w_j} - \frac{1}{f_i w_i}\right) \\
    = & \frac{1}{\alpha t}\left( \frac{1}{m}\sum_{i=1}^m(w_i f_i)^p\sum_{i=1}^m\frac{1}{w_i f_i} - \right. \\
    & \left.\sum_{i=1}^m(w_i f_i)^p\frac{1}{w_i f_i} \right).
\end{aligned}
\end{equation}
We can assume $w_1 f_1\!\geq\!\ldots\!\geq\! w_m f_m$, $\tilde{w}_1 \tilde{f}_1\leq\ldots\leq \tilde{w}_m \tilde{f}_m$ and $w_i f_i=\tilde{w}_{m-i+1} \tilde{f}_{m-i+1}$ for $i=1,\ldots,m$. According to rearrangement inequality and Tchebycheff's sum inequality, the following inequality holds
\begin{equation}
\begin{aligned}
    \sigma_4-\sigma_5
    \geq & \frac{1}{\alpha t}\left( \frac{1}{m}\sum_{i=1}^m(w_i f_i)^p\sum_{i=1}^m\frac{1}{w_i f_i} - \right. \\
    & \left.\sum_{i=1}^m(w_i f_i)^p\frac{1}{\tilde{w_i}\tilde{f_i}} \right) \\
    \geq & \frac{1}{\alpha t}\left( \frac{1}{m}\sum_{i=1}^m(w_i f_i)^p\sum_{i=1}^m\frac{1}{w_i f_i} - \right. \\
    & \left.\frac{1}{m}\sum_{i=1}^m(w_i f_i)^p\sum_{i=1}^m\frac{1}{\tilde{w_i}\tilde{f_i}} \right)=0. \\
\end{aligned}
\end{equation}
Since $\frac{1}{\alpha}+v_i f_i t = w_i f_i \geq 0$, $\sigma_3\geq 0$. Therefore, Eq. \eqref{eqn:Mp_v_d1} $\geq 0$ which represents that $g^{glp}(\mathbf{w}|\mathbf{f},\mathbf{z}^*)$ is unimodal. Let $\sigma_3=0$, we can get $w_i f_i=\frac{1}{\alpha}+v_i f_i t=0$ for $i=1,\ldots,m$. In this case, $\sigma_4=0$ and $\sigma_5=0$.
$\sigma_4-\sigma_5=0$ if and only if $w_i f_i = w_j f_j,~i,j\in\{1,\ldots,m\}$. Then the global minimal solution is $w_i=\frac{1}{\alpha f_i} \mbox{ for } i=1,\ldots,m$.
\end{proof}

\begin{table*}[!ht]\small
  \centering
    \scalebox{1}{
    \setlength{\tabcolsep}{1mm}{
    \begin{tabular}{ccc|ccc|ccc|ccc}
    \toprule
    \multirow{3}[4]{*}{Problem} & \multirow{3}[4]{*}{$m$} & \multirow{3}[4]{*}{$I_H$} & \multicolumn{3}{c|}{$p=1$} & \multicolumn{3}{c|}{$p=2$} & \multicolumn{3}{c}{$p\rightarrow\infty$} \\
\cmidrule{4-12}          &       &       & \multirow{2}[2]{*}{MOEA/D} & MOEA/D & MOEA/D & \multirow{2}[2]{*}{MOEA/D} & MOEA/D & MOEA/D & \multirow{2}[2]{*}{MOEA/D} & MOEA/D & MOEA/D \\
          &       &       &       & -GR   & -GGR  &       & -GR   & -GGR  &       & -GR   & -GGR \\
    \midrule
    \multirow{2}[1]{*}{ZDT1} & \multirow{2}[1]{*}{2} & mean  & 0.8614(2)- & 0.0097(3)- & \cellcolor{gray} 0.8657(1) & \cellcolor{gray} 0.871(1)+ & 0.8398(3)- & 0.869(2) & 0.8655(3)- & \cellcolor{gray} 0.8699(1)= & 0.8698(2) \\
          &       & std.  & 7.4e-04 & 2.3e-02 & 4.0e-03 & 2.0e-04 & 5.9e-03 & 2.1e-03 & 2.0e-03 & 2.2e-03 & 1.9e-03 \\
    \multirow{2}[0]{*}{ZDT2} & \multirow{2}[0]{*}{2} & mean  & 0.21(2)- & 0(3)- & \cellcolor{gray} 0.4436(1) & \cellcolor{gray} 0.5174(1)= & 0.4507(3)- & 0.5108(2) & 0.5327(3)- & 0.5343(2)= & \cellcolor{gray} 0.5348(1) \\
          &       & std.  & 7.4e-12 & 0.0e+00 & 4.0e-02 & 2.6e-03 & 4.4e-02 & 1.4e-02 & 1.6e-03 & 6.7e-03 & 6.2e-03 \\
    \multirow{2}[0]{*}{ZDT3} & \multirow{2}[0]{*}{2} & mean  & 0.4873(2)- & 0.03427(3)- & \cellcolor{gray} 0.7154(1) & 0.6826(2)- & 0.6824(3)- & \cellcolor{gray} 0.719(1) & 0.7173(3)- & 0.7215(2)- & \cellcolor{gray} 0.7217(1) \\
          &       & std.  & 9.6e-02 & 3.7e-02 & 2.2e-03 & 3.7e-02 & 2.6e-02 & 1.6e-03 & 1.2e-02 & 3.7e-03 & 3.8e-03 \\
    \multirow{2}[1]{*}{ZDT4} & \multirow{2}[1]{*}{2} & mean  & 0.845(2)= & 0(3)- & \cellcolor{gray} 0.8487(1) & 0.8599(2)= & 0.8051(3)- & \cellcolor{gray} 0.8599(1) & 0.8431(3)- & 0.8566(2)= & \cellcolor{gray} 0.857(1) \\
          &       & std.  & 1.2e-02 & 0.0e+00 & 8.8e-03 & 5.2e-03 & 4.8e-02 & 4.8e-03 & 1.4e-02 & 8.5e-03 & 9.4e-03 \\
    \midrule
    \multirow{6}[2]{*}{DTLZ1} & \multirow{2}[1]{*}{2} & mean  & 0.2099(2)- & 0(3)- & \cellcolor{gray} 0.7027(1) & 0.7025(2)- & 0.6875(3)- & \cellcolor{gray} 0.7041(1) & 0.7045(3)- & 0.7049(2)= & \cellcolor{gray} 0.7049(1) \\
          &       & std.  & 1.3e-04 & 0.0e+00 & 1.6e-03 & 2.9e-04 & 1.3e-04 & 7.2e-05 & 3.8e-04 & 1.6e-04 & 1.7e-04 \\
          & \multirow{2}[0]{*}{3} & mean  & 0.4923(2)- & 0.2171(3)- & \cellcolor{gray} 1.033(1) & \cellcolor{gray} 1.115(1)+ & 1.078(3)- & 1.113(2) & \cellcolor{gray} 1.101(1)+ & 1.1(2)= & 1.1(3) \\
          &       & std.  & 7.6e-02 & 1.8e-01 & 8.3e-03 & 2.2e-04 & 7.9e-04 & 2.1e-04 & 6.3e-04 & 2.9e-04 & 2.0e-04 \\
          & \multirow{2}[1]{*}{5} & mean  & 0.8878(2)- & 0.7822(3)- & \cellcolor{gray} 1.318(1) & 1.511(3)- & \cellcolor{gray} 1.519(1)+ & 1.516(2) & 1.511(3)- & \cellcolor{gray} 1.513(1)+ & 1.512(2) \\
          &       & std.  & 1.2e-01 & 3.8e-01 & 6.3e-02 & 7.1e-04 & 4.0e-03 & 2.2e-03 & 9.9e-04 & 1.3e-03 & 1.6e-03 \\
    \midrule
    \multirow{6}[2]{*}{DTLZ3} & \multirow{2}[1]{*}{2} & mean  & 0.2099(2)- & 0.00818(3)- & \cellcolor{gray} 0.4138(1) & 0.2099(3)- & 0.4164(2)- & \cellcolor{gray} 0.4171(1) & 0.4197(3)- & \cellcolor{gray} 0.42(1)= & 0.42(2) \\
          &       & std.  & 2.1e-04 & 2.1e-02 & 1.8e-02 & 7.5e-05 & 1.1e-02 & 1.2e-04 & 3.4e-04 & 1.1e-04 & 2.1e-04 \\
          & \multirow{2}[0]{*}{3} & mean  & 0.3307(2)- & 0.1788(3)- & \cellcolor{gray} 0.6244(1) & 0.3317(3)- & 0.6016(2)- & \cellcolor{gray} 0.6508(1) & \cellcolor{gray} 0.7344(1)+ & 0.733(2)= & 0.7329(3) \\
          &       & std.  & 3.9e-04 & 1.4e-01 & 5.6e-03 & 4.0e-03 & 8.8e-03 & 5.2e-03 & 1.3e-03 & 7.3e-04 & 7.9e-04 \\
          & \multirow{2}[1]{*}{5} & mean  & 0.6102(2)- & 0.5385(3)- & \cellcolor{gray} 0.6111(1) & 0.6208(3)- & 0.6493(2)- & \cellcolor{gray} 0.8244(1) & 1.146(3)- & \cellcolor{gray} 1.149(1)= & 1.148(2) \\
          &       & std.  & 1.1e-03 & 4.6e-02 & 8.9e-04 & 1.7e-02 & 7.3e-02 & 6.6e-02 & 1.1e-03 & 3.0e-03 & 2.2e-03 \\
    \midrule
    \multirow{4}[2]{*}{DTLZ5} & \multirow{2}[1]{*}{3} & mean  & 0.131(3)- & 0.161(2)- & \cellcolor{gray} 0.232(1) & 0.131(3)- & 0.2246(2)- & \cellcolor{gray} 0.2331(1) & \cellcolor{gray} 0.2644(1)+ & 0.2556(3)- & 0.2559(2) \\
          &       & std.  & 3.2e-10 & 9.4e-03 & 9.8e-04 & 6.6e-06 & 1.5e-03 & 1.6e-03 & 4.1e-06 & 1.4e-05 & 5.1e-05 \\
          & \multirow{2}[1]{*}{5} & mean  & 0.1464(3)- & \cellcolor{gray} 0.1524(1)+ & 0.1481(2) & 0.1464(3)- & 0.1475(2)- & \cellcolor{gray} 0.1504(1) & \cellcolor{gray} 0.1925(1)+ & 0.15(3)- & 0.1894(2) \\
          &       & std.  & 5.0e-04 & 4.0e-03 & 1.6e-03 & 4.1e-04 & 2.5e-03 & 1.6e-03 & 5.1e-04 & 1.0e-03 & 4.2e-04 \\
    \midrule
    \multirow{4}[2]{*}{MOKP} & \multirow{2}[1]{*}{2} & mean  & 0.8719(2)= & 0.2195(3)- & \cellcolor{gray} 0.8748(1) & \cellcolor{gray} 0.8731(1)= & 0.8227(3)- & 0.8691(2) & 0.8384(3)- & \cellcolor{gray} 0.8503(1)= & 0.8495(2) \\
          &       & std.  & 1.1e-02 & 9.3e-02 & 8.8e-03 & 8.8e-03 & 1.4e-02 & 9.5e-03 & 1.3e-02 & 1.2e-02 & 1.2e-02 \\
          & \multirow{2}[1]{*}{3} & mean  & 0.6274(2)- & 0.1097(3)- & \cellcolor{gray} 0.6561(1) & 0.6588(2)- & 0.636(3)- & \cellcolor{gray} 0.6657(1) & 0.6239(3)- & 0.636(2)= & \cellcolor{gray} 0.6366(1) \\
          &       & std.  & 1.1e-02 & 2.2e-02 & 6.1e-03 & 7.2e-03 & 7.1e-03 & 6.5e-03 & 7.1e-03 & 5.6e-03 & 5.4e-03 \\
    \midrule
    \multirow{4}[2]{*}{MOTSP} & \multirow{2}[1]{*}{2} & mean  & 0.9733(2)- & 0.03373(3)- & \cellcolor{gray} 0.9852(1) & 0.9815(2)= & 0.9491(3)- & \cellcolor{gray} 0.9861(1) & 0.9392(3)- & 0.9679(2)= & \cellcolor{gray} 0.9711(1) \\
          &       & std.  & 1.5e-02 & 4.7e-02 & 1.0e-02 & 1.1e-02 & 1.5e-02 & 9.5e-03 & 1.2e-02 & 1.2e-02 & 1.4e-02 \\
          & \multirow{2}[1]{*}{3} & mean  & 0.9253(2)= & 0.2679(3)- & \cellcolor{gray} 0.929(1) & 0.9016(2)- & 0.8854(3)- & \cellcolor{gray} 0.9131(1) & 0.8174(3)- & \cellcolor{gray} 0.8503(1)= & 0.8452(2) \\
          &       & std.  & 1.3e-02 & 3.8e-02 & 1.1e-02 & 1.4e-02 & 1.2e-02 & 1.1e-02 & 1.5e-02 & 1.6e-02 & 1.7e-02 \\
    \midrule
    \multicolumn{3}{c|}{Total +/-/=} & 0/13/3 & 1/15/0 & \textbackslash{} & 2/10/4 & 1/15/0 & \textbackslash{} & 4/12/0 & 1/3/13 & \textbackslash{} \\
    \bottomrule
    \end{tabular}%
    }
    }
  \caption{Mean and standard deviation of $I_H$ metric values. The rank of each algorithm on each instance is provided after the mean of the $I_H$ metric value. +, - or = denotes that the performance of the corresponding algorithm is statistically better than, worse than or similar to that of MOEA/D-GGR based on Wilcoxon's rank sum test at 0.05 significant level.}
  \label{tab:baseline}%
\end{table*}%


The preference regions of the G$L_p$-based subproblems are illustrated in \figurename~\ref{fig:preference_region_GNS}. In the 2-objective case, the preference regions are almost the same regardless of how $p$-value varies. But the preference region sizes of boundary subproblems are quite small. To cope with this problem, MOEA/D-GGR\footnote{\url{https://github.com/EricZheng1024/MOEA-D-GGR}} is proposed. The difference between MOEA/D-GGR and MOEA/D-GR is the scalarization method and the replacement neighborhood. After determining the replacement neighborhood as MOEA/D-GR does, each boundary subproblem is added to the neighborhood of the closest non-boundary subproblem in MOEA/D-GGR. This modification makes each boundary subproblem's solution have a higher updating probability without degrading the preference regions of other subproblems.

\section{Experimental Studies}
\subsection{Experimental Setup}
\subsubsection{Test Instances.} We use ZDT1-ZDT4~\cite{zitzler2000comparison}, DTLZ1, DTLZ3 and DTLZ5~\cite{deb2005scalable}, the multi-objective knapsack problem (MOKP)~\cite{zitzler1999multiobjective}, and the multi-objective traveling salesman problem (MOTSP)~\cite{corne2007techniques} to verify the algorithm performance. We set the decision vector dimension to 30, 30, 30 and 10 for ZDT1-ZDT4 respectively, and to $(m+4)$ for each DTLZ instance. The MOKP instances are randomly generated with 250 items. The MOTSP instances are randomly generated with 60 vertexes. Each solution is encoded by real numbers for the ZDT and DTLZ instances, by binary numbers for MOKP, and by a permutation for MOTSP.

\subsubsection{General Algorithm Settings.} MOEA/D and MOEA/D-GR are compared with MOEA/D-GGR. All the algorithms are implemented on the PlatEMO platform~\cite{tian2017platemo}. The detailed parameter settings are as follows:
\begin{itemize}
    \item The population size $N$: 100 ($m=2$) or 190 ($m=3$).
    \item The maximal number of function evaluations: 25000 for ZDT1-ZDT4, 100000 for DTLZ1, DTLZ3 and DTLZ5, 200000 for 2-objective MOKP and MOTSP, and 400000 for 3-objective MOKP and MOTSP.
    \item The number of independent runs: 30 for each instance.
    \item The neighborhood size: $T_m=0.1N$ and $T_r=\lceil 0.05N\rceil$.
    \item Reproduction operators: For real number coding, the simulated binary crossover (SBX) and polynomial mutation (PM) are used~\cite{purshouse2007evolutionary}. The SBX control parameters $p_c$, $\eta_c$ and $p_e$ are set to $1$, $20$ and $0$, respectively. The PM control parameters $p_m$ and $\eta_m$ are set to $1/n$ and $20$, where $n$ is the number of decision variables. For binary coding, the uniform crossover and bit-flip mutation are used~\cite{syswerda1989uniform}. The crossover rate is $1$; the mutation rate is $2/n$ for a bit. For permutation coding, the order-based crossover and simple inversion mutation are used~\cite{larranaga1999genetic}. The crossover rate is $1$ and the mutation rate is $0.1$.
\end{itemize}


\subsubsection{Performance Metric.}
The hypervolume indicator ($I_H$) is used to assess the performance of each algorithm~\cite{zitzler1999multiobjective}. Let $P$ be an approximate solution set and $\mathbf{z}^h=(z^h_1,\ldots,z^h_m)^{\intercal}$ be a reference objective vector. The $I_H$ metric value is computed by
\begin{equation}\label{eqn:hv}
    I_H(P,\!\mathbf{z}^h) \!=\! vol \!\left( \bigcup_{\mathbf{x}\in P} [f_1(\mathbf{x}),\!z_1^h]\!\!\times\!\!\ldots\!\!\times\!\![f_m(\mathbf{x}),\!z_m^h] \right)\!,
\end{equation}
where $vol(\cdot)$ is the Lebesgue measure.
A larger $I_H$ metric value indicates a better algorithm performance. Before calculate the $I_H$ metric value, we normalize $\{f_i(\mathbf{x})|x\in P\}$ for $i=1,\ldots,m$ with the range from $\min\{f_i(\mathbf{x})|\mathbf{x}\in PS\}$ to $\max\{f_i(\mathbf{x})|\mathbf{x}\in PS\}$. Then set $z_i^h = 1.1$ for $i=1,\ldots,m$. The $PF$s are unknown on the MOKP and MOTSP instances. Each of them is approximated by the set of all non-dominated solutions obtained by all algorithms in all runs.

\begin{figure}[t]
\centering
\subfigure{
    \includegraphics[width=0.14\textwidth]{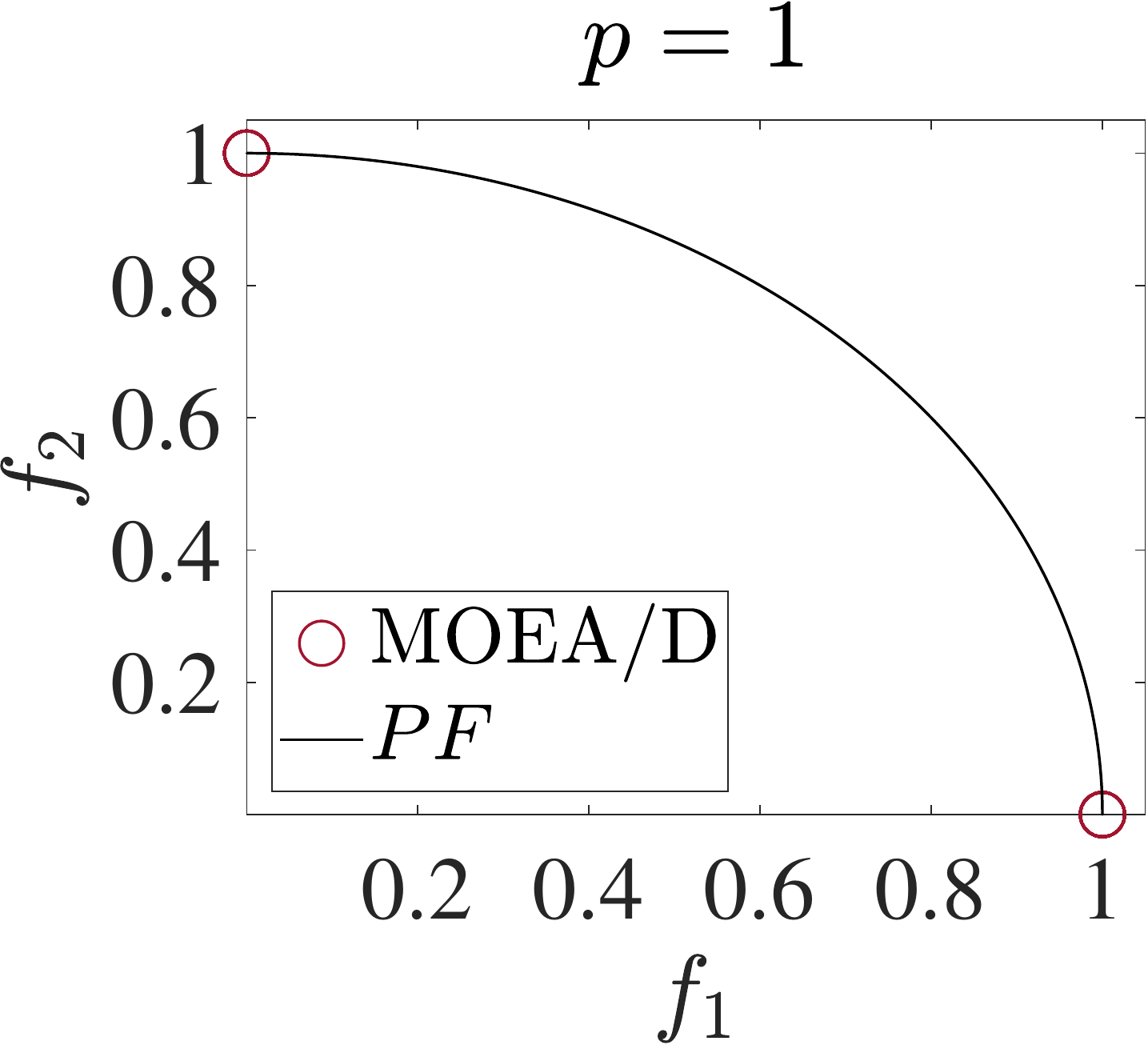}
}
\subfigure{
    \includegraphics[width=0.14\textwidth]{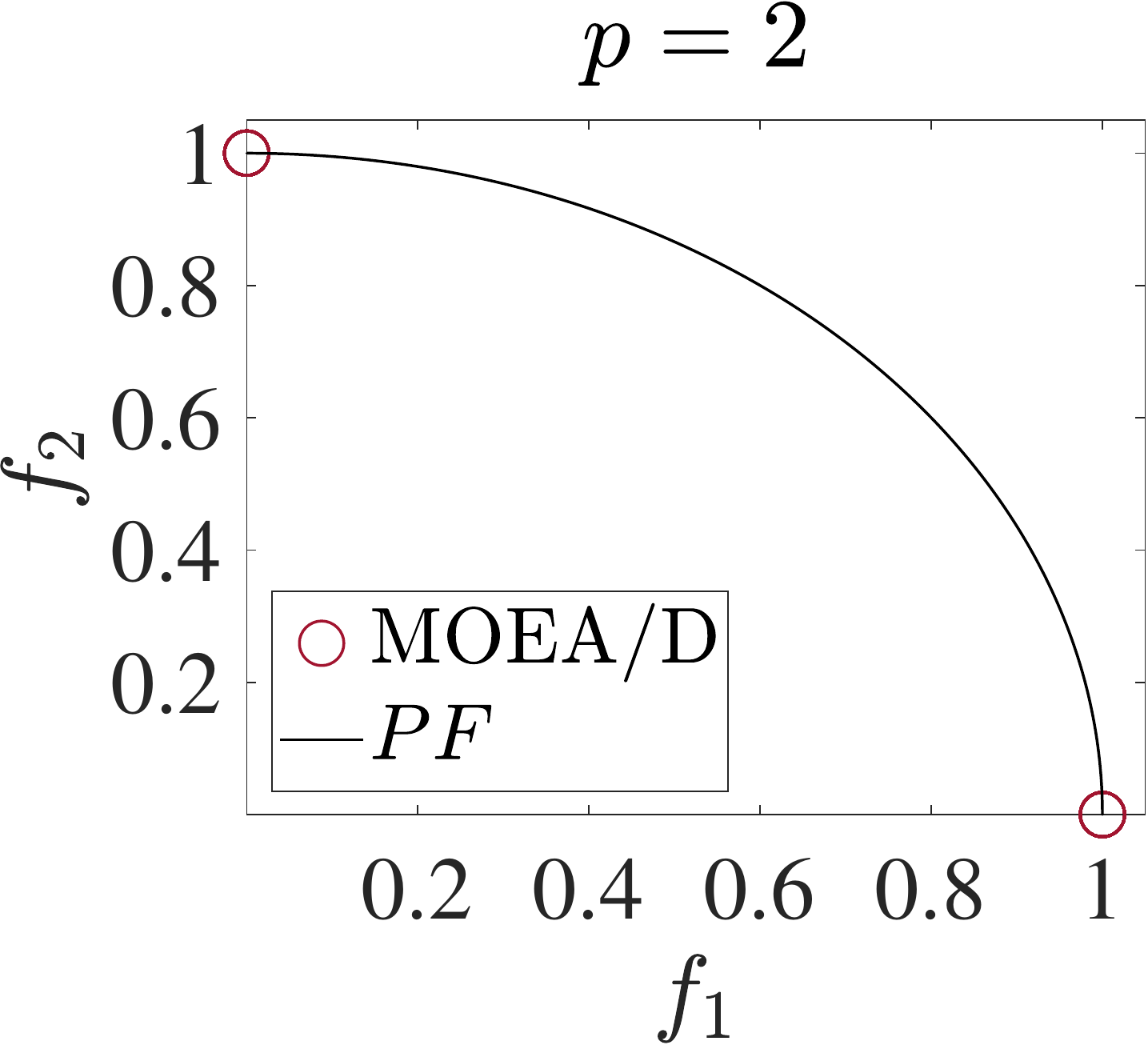}
}
\subfigure{
    \includegraphics[width=0.14\textwidth]{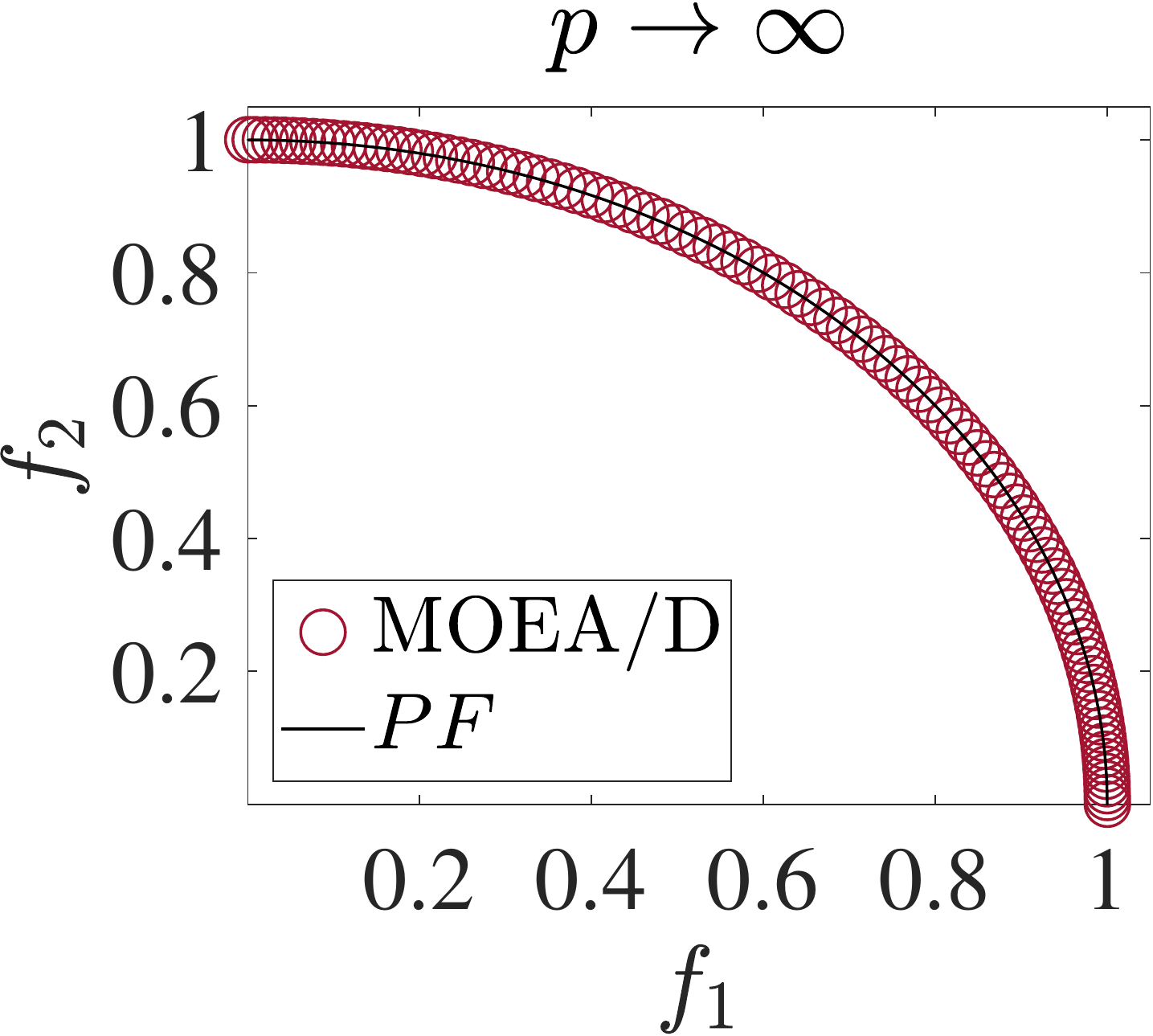}
}
\subfigure{
    \includegraphics[width=0.14\textwidth]{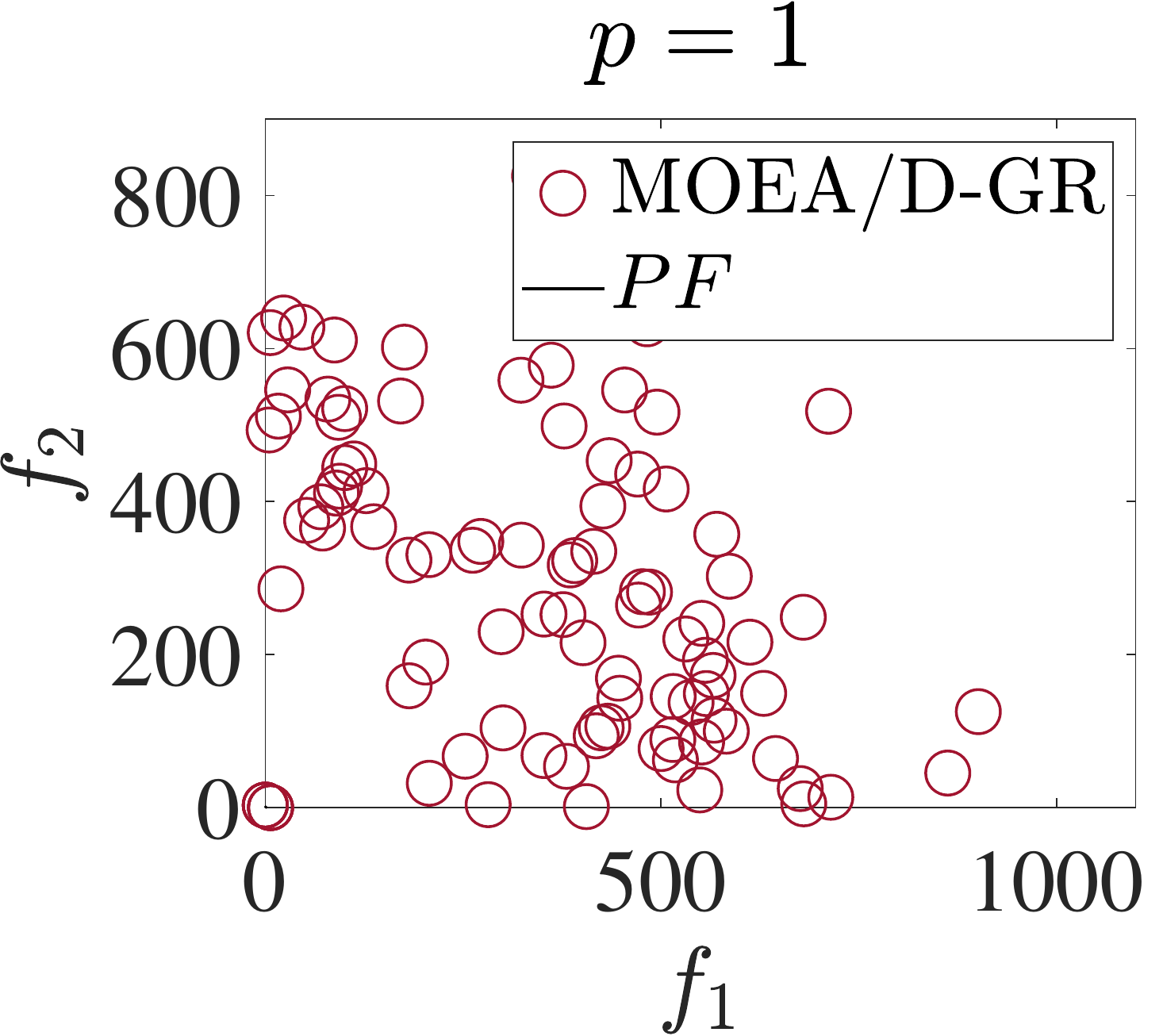}
}
\subfigure{
    \includegraphics[width=0.14\textwidth]{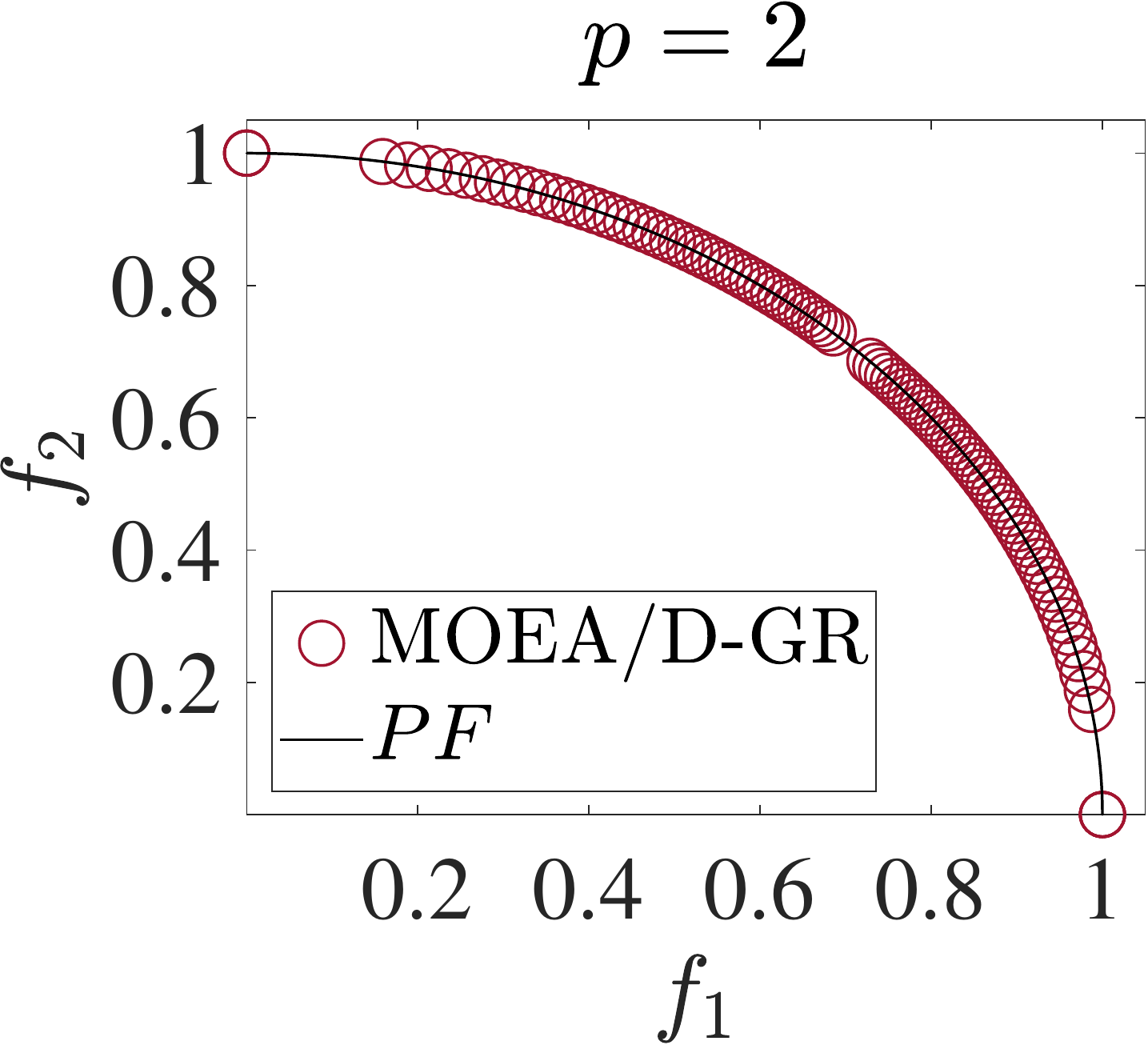}
}
\subfigure{
    \includegraphics[width=0.14\textwidth]{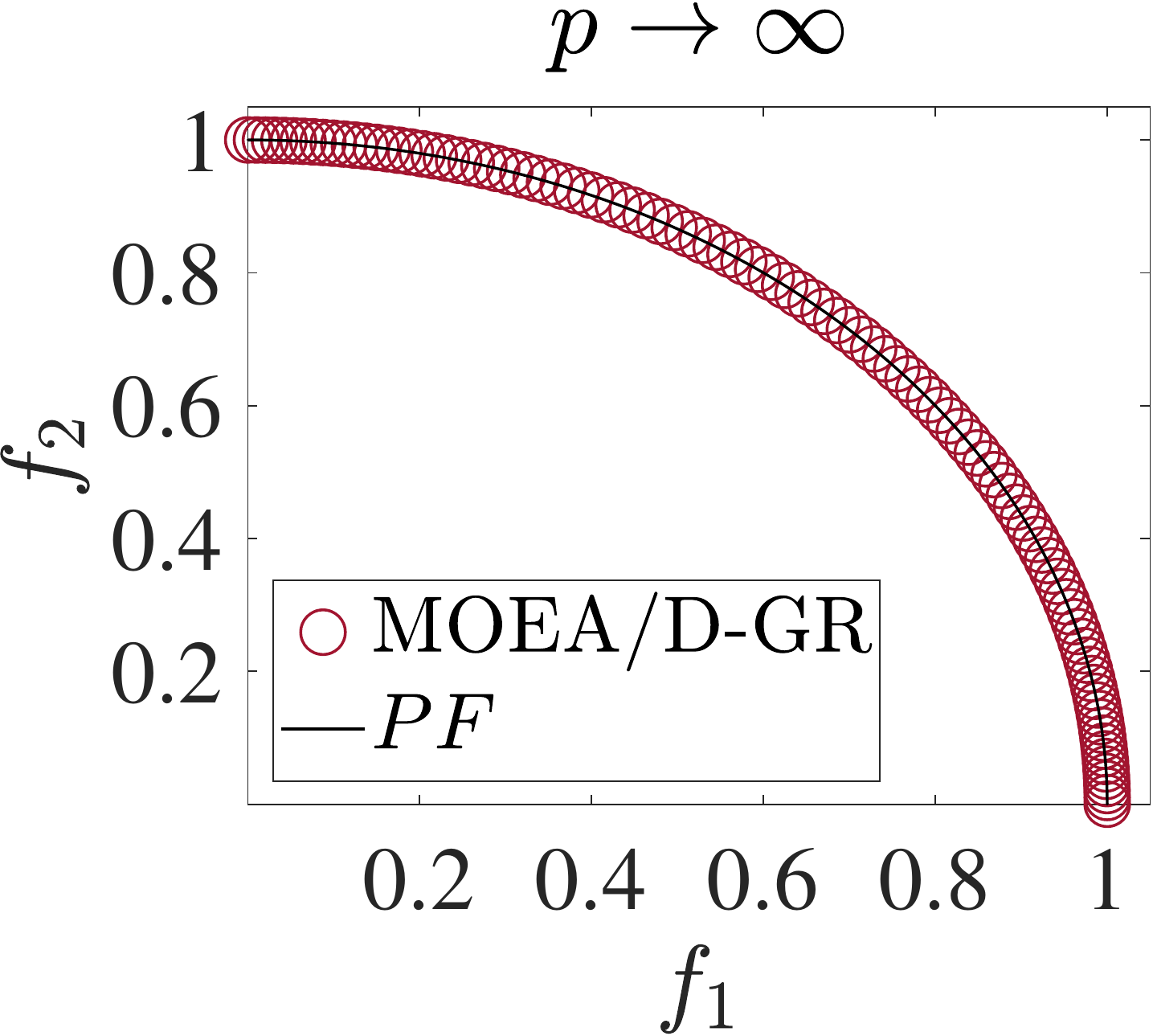}
}
\subfigure{
    \includegraphics[width=0.14\textwidth]{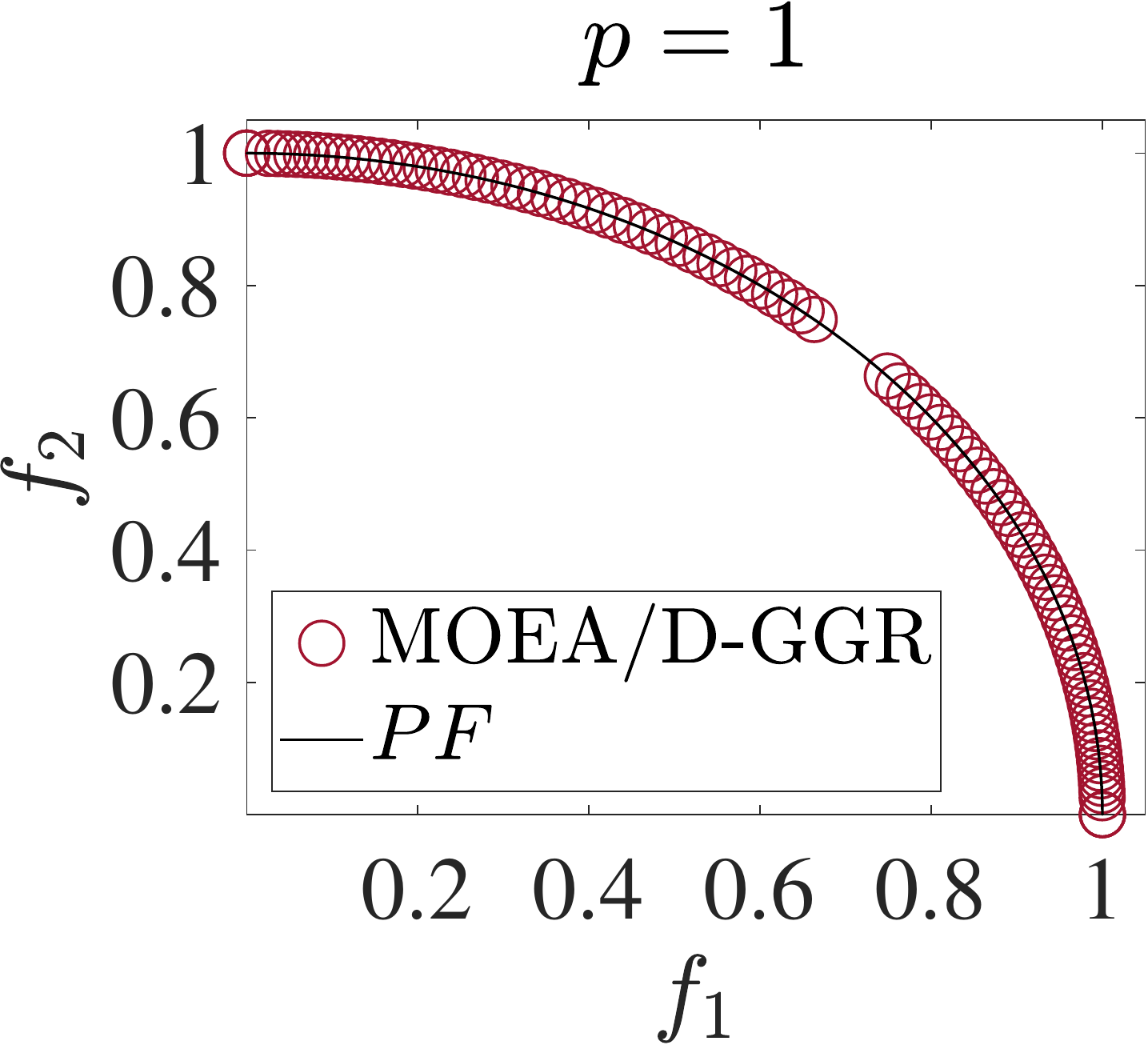}
}
\subfigure{
    \includegraphics[width=0.14\textwidth]{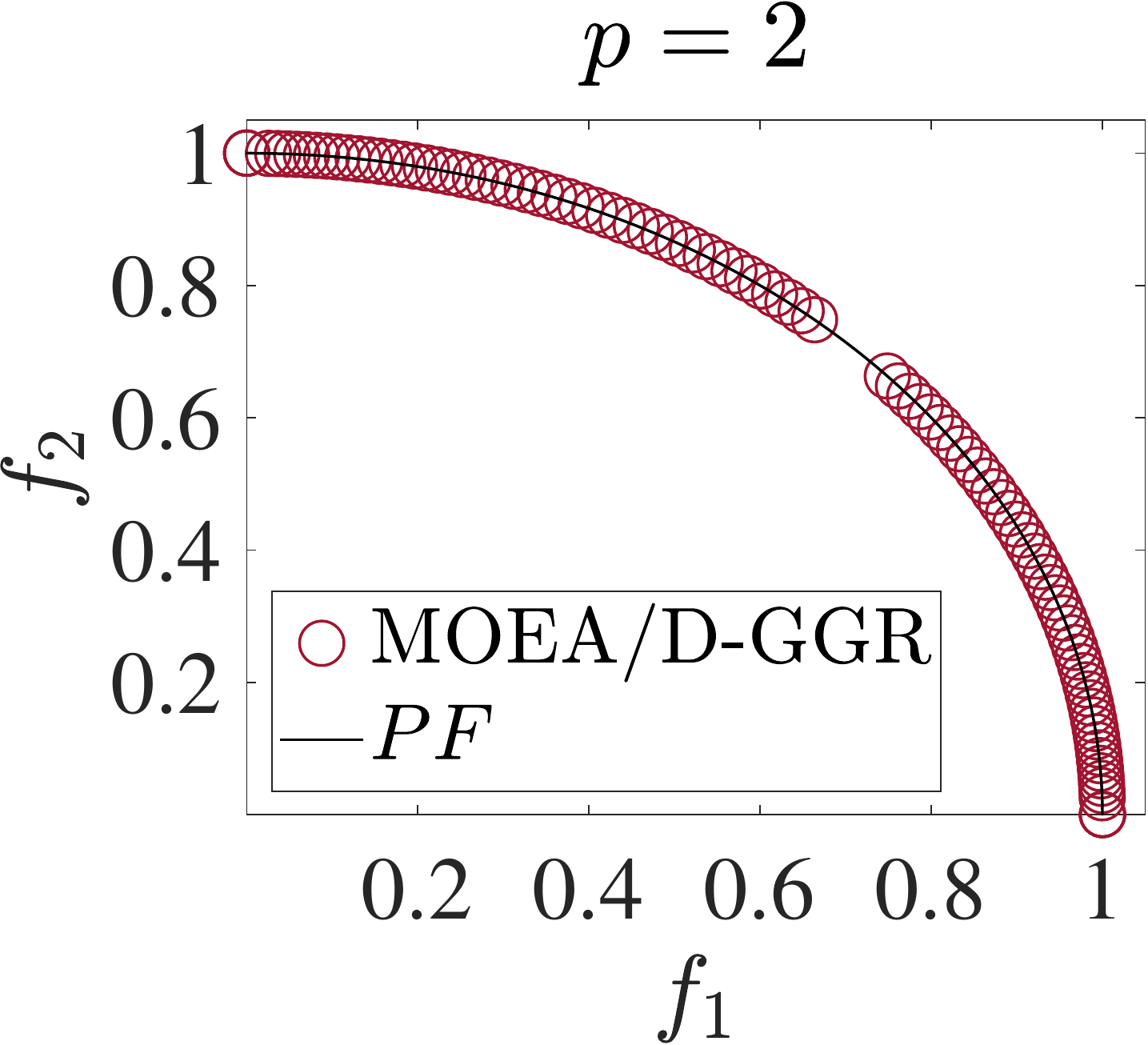}
}
\subfigure{
    \includegraphics[width=0.14\textwidth]{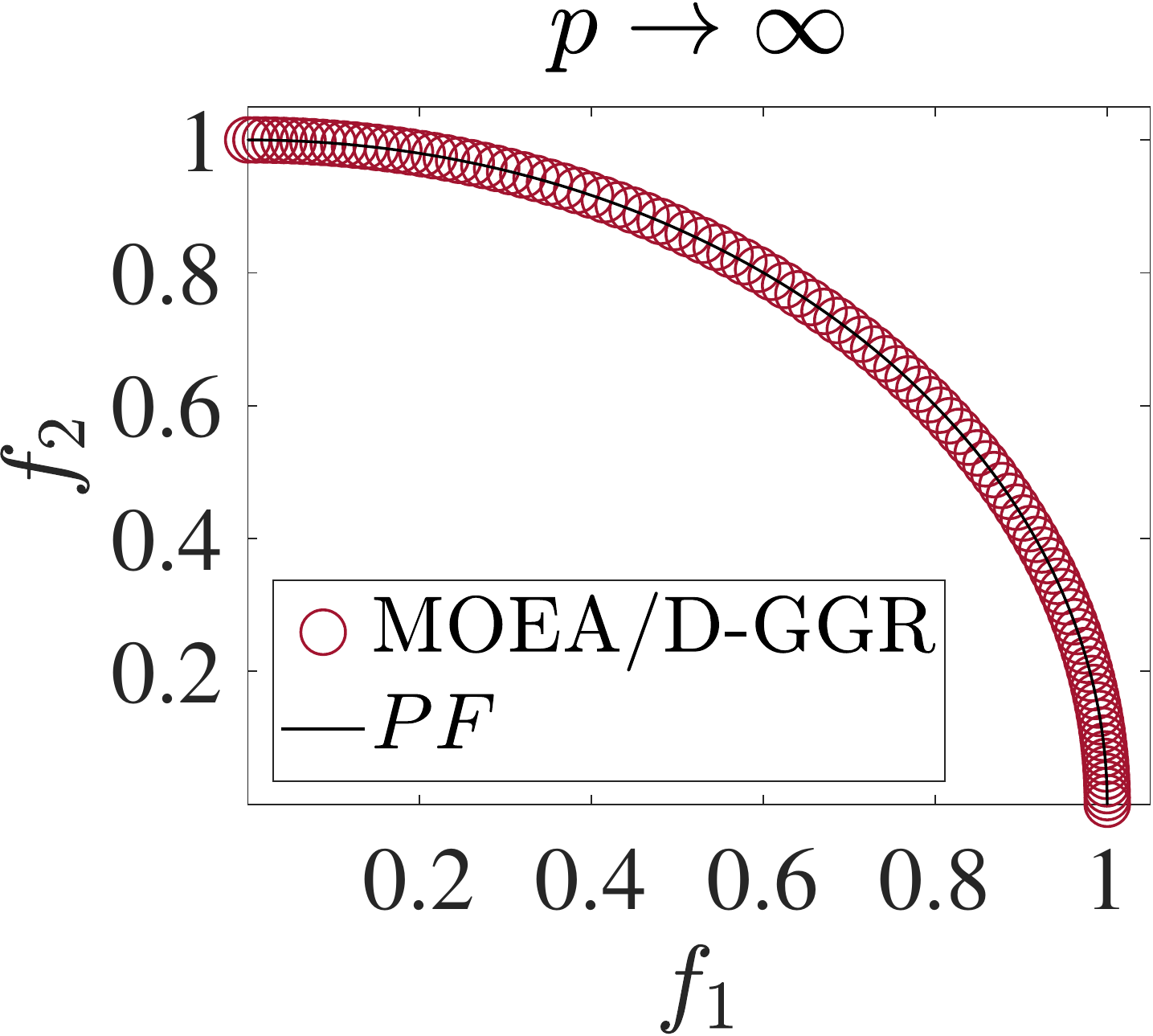}
}
\caption{Plots of the objective vectors having the median $I_H$ metric value obtained among each algorithm's 30 runs on 2-objective DTLZ3.}
\label{fig:pop}
\end{figure}

\subsection{Experimental Results}
The $I_H$ metric values obtained by the three algorithms on 16 test instances are given in Tables \ref{tab:baseline}. Our theoretical analysis holds on the problem with many-objective (\eg, DTLZ1, DTLZ3 and DTLZ5 with 5 objectives), discrete objective space (\eg, MOKP and MOTSP), convex $PF$ (\eg, ZDT1 and ZDT4), concave $PF$ (\eg, ZDT2 and DTLZ3), linear $PF$ (\eg, DTLZ1), discontinuous $PF$ (\eg, ZDT3). 

When $p\rightarrow\infty$, MOEA/D-GGR and MOEA/D-GR have similar performance; MOEA/D is slightly worse than the two algorithms. When $p=2$, MOEA/D-GGR is better than the other competitors on most test problems. When $p=1$, MOEA/D-GGR significantly outperforms MOEA/D and MOEA/D-GR on all test problems except 5-objective DTLZ5. This instance has a highly degenerate $PF$. The direction vectors have few intersections with such a $PF$, making our mismatch avoidance strategy ineffective.

\figurename~\ref{fig:pop} plots the obtained final objective vectors which have the median $I_H$ metric values from each algorithm's 30 runs on 2-objective DTLZ3. When $p=1$, MOEA/D-GGR is the only algorithm that overcomes mismatches, which achieves a satisfactory approximation to the $PF$; MOEA/D and MOEA/D-GR fail to approximate the $PF$. When $p=2$, MOEA/D still can only obtain the two boundary objective vectors; the objective vectors obtained by MOEA/D-GR cover the middle part but miss some boundary parts of the $PF$; MOEA/D-GGR yields the best approximation of the $PF$ compared to the other two algorithms. When $p\rightarrow\infty$, the three algorithms have similar good performance.

It is worth mentioning that the approximate set obtained by MOEA/D-GGR using G$L_1$ or G$L_2$ misses a small central part of the $PF$. It is because the diversity of MOEA/D-GGR is affected by the setting of $T_r$ when G$L_p$ with $1\leq p <\infty$ is used. When a large $T_r$ is adopted, MOEA/D-GGR using G$L_p$ with $1\leq p <\infty$ may miss some non-dominated objective vectors. In other words, its diversity can be improved by employing a small $T_r$. To validate it, we set $T_r$ to 1 for MOEA/D-GGR with G$L_1$ and G$L_2$ and further test the two algorithms on the 2-objective DTLZ3. As shown in \figurename~\ref{fig:pop_Tr1}, the obtained approximate set has much better diversity in both cases. But \figurename~\ref{fig:conv} also indicates that a small $T_r$ may discourage convergence. As reported in~\cite{wang2015adaptive}, a better approach is to adaptively change $T_r$ during the search.

Overall, MOEA/D-GGR is less affected by the $PF$ shape than MOEA/D and MOEA/D-GR.

\begin{figure}[t]
\centering
\subfigure{
    \includegraphics[width=0.18\textwidth]{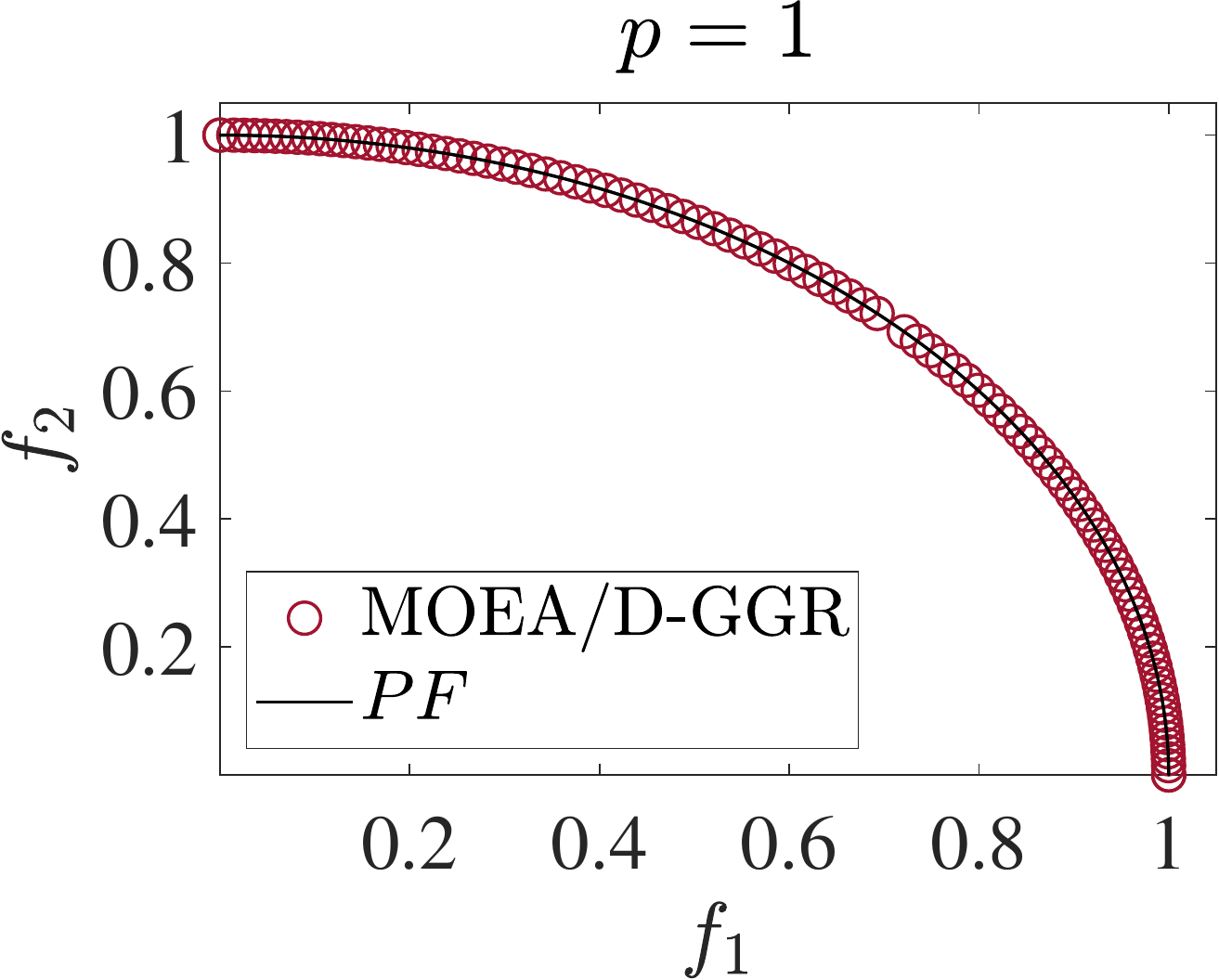}
}
\subfigure{
    \includegraphics[width=0.18\textwidth]{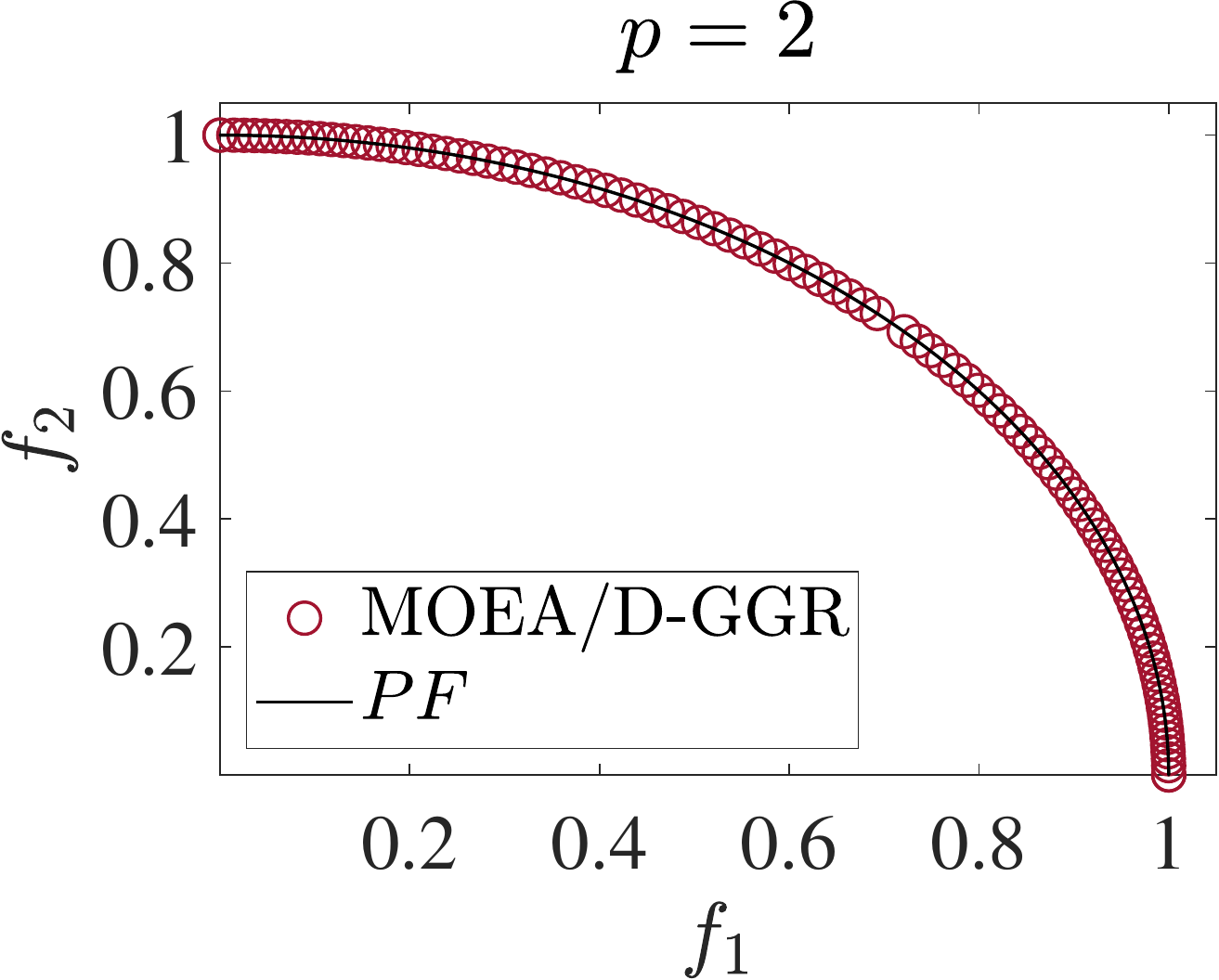}
}
\caption{Plots of the objective vectors having the median $I_H$ metric value obtained by MOEA/D-GGR ($p=1$ and $p=2$) using $T_r=1$ on 2-objective DTLZ3.}
\label{fig:pop_Tr1}
\end{figure}

\begin{figure}[t]
\centering
\subfigure{
    \includegraphics[width=0.18\textwidth]{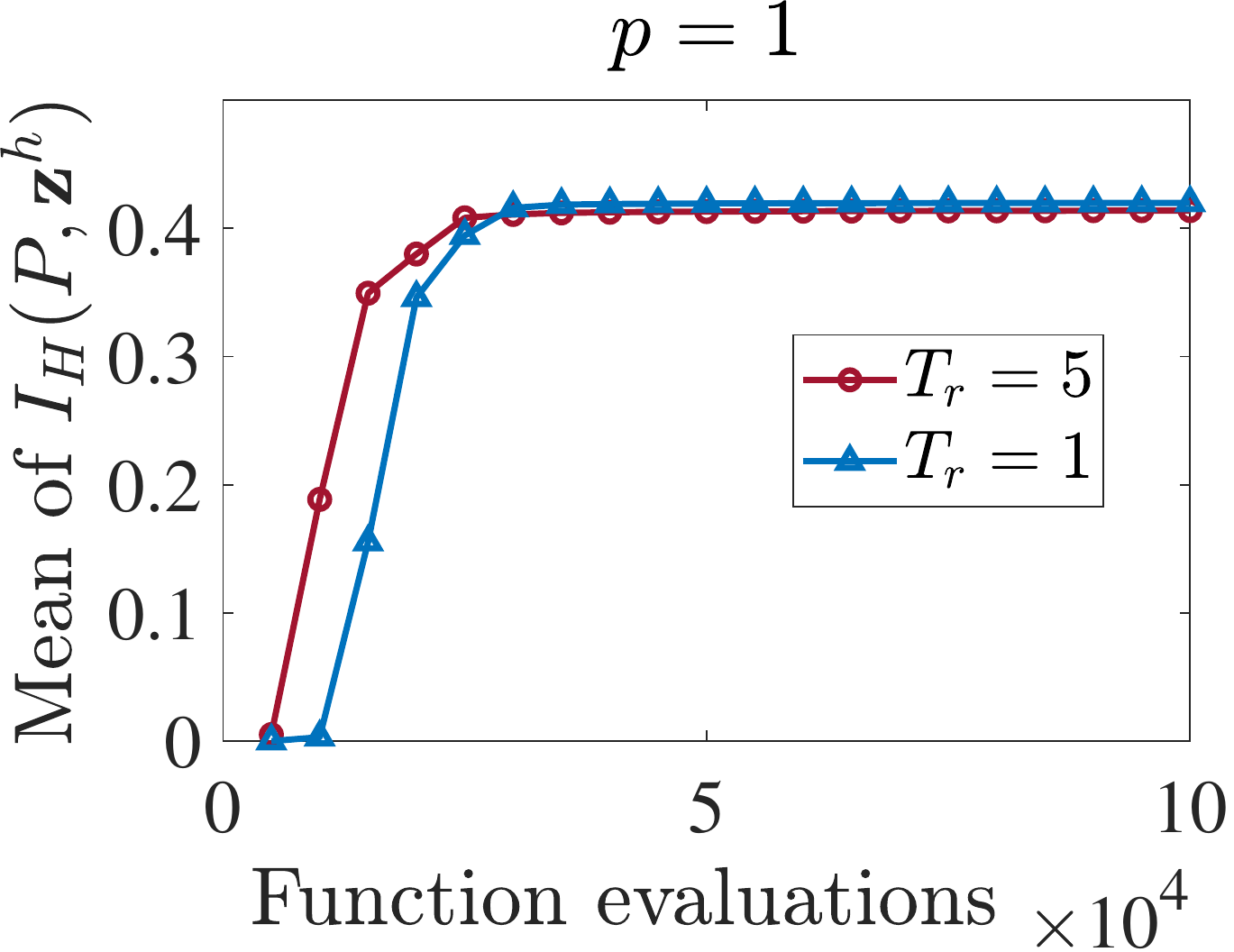}
}
\subfigure{
    \includegraphics[width=0.18\textwidth]{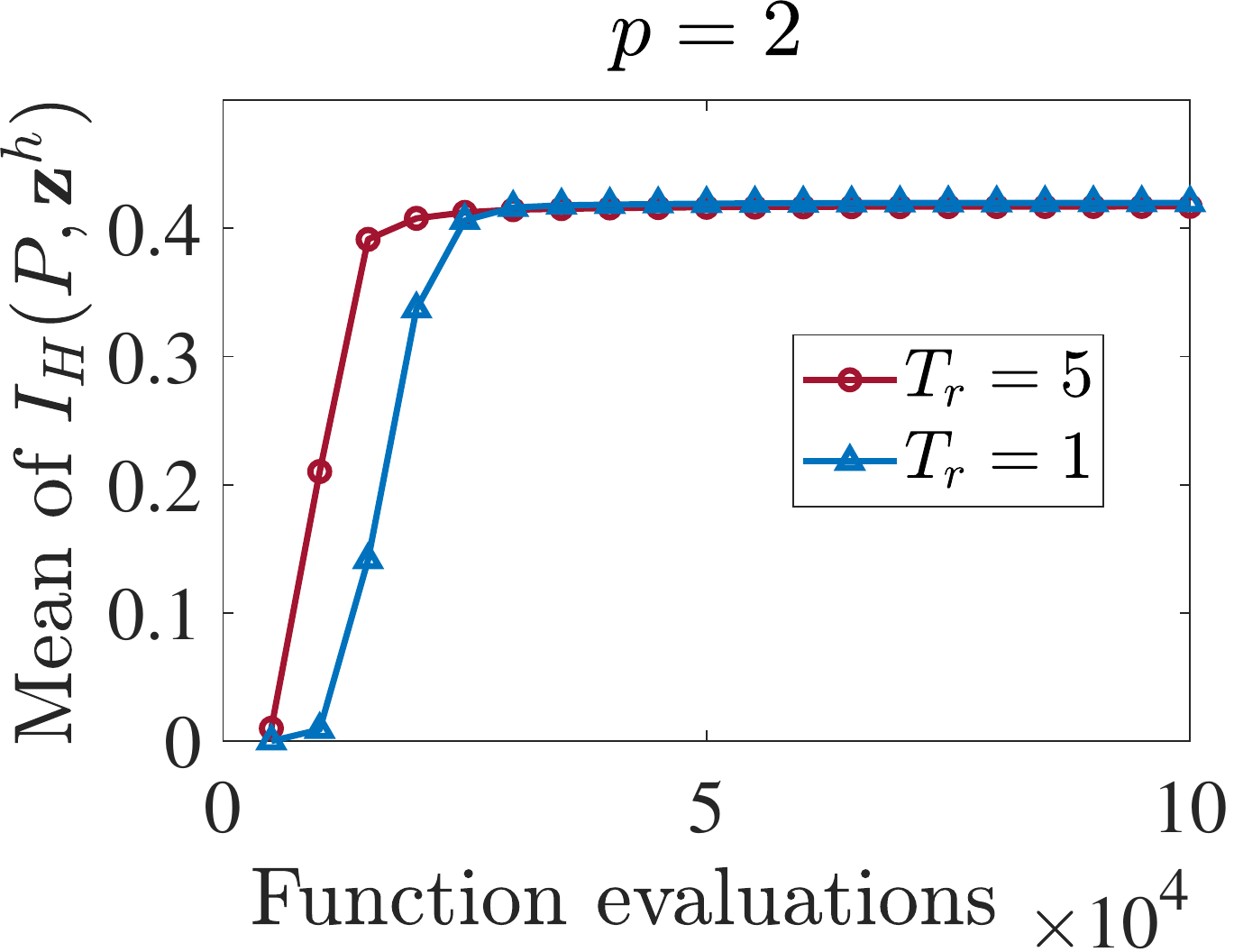}
}
\caption{Evolution of the mean $I_H$ metric values achieved by MOEA/D-GGR using two different $T_r$ settings on 2-objective DTLZ3.}
\label{fig:conv}
\end{figure}


\section{Conclusion}

In this paper, we have demonstrated that MOEA/D-GR still suffers from mismatches when the $L_{\infty}$ scalarization is replaced by another $L_{p}$ scalarization with $p\in [1,\infty)$. Our analysis reveals that this can be attributed to $L_p$-based ($1\leq p<\infty$) subproblems having inconsistently large preference regions. When $p$ is set to a small value, some middle subproblems have very small preference regions so that their direction vectors cannot pass through their corresponding preference regions. To fill this gap, we have proposed a new scalarization family called the G$L_p$ scalarization. The G$L_{p}$-based subproblem's direction vector is guaranteed to pass through its corresponding preference region, which implies MOEA/D-GR can always avoid mismatches when using the G$L_p$ scalarization for any $p\geq 1$. We have conducted various experimental studies to validate the effectiveness of the G$L_p$ scalarization.




\section{Acknowledgments}
This work was supported in part by the National Natural Science Foundation of China under Grant 62106096; and in part by the Shenzhen Technology Plan under Grant JCYJ20220530113013031.

\bibliography{aaai23.bib}


\clearpage
\appendix

\end{document}